\documentclass{article}
\usepackage[margin=2.5cm]{geometry}
\usepackage[numbers]{natbib}
\usepackage{cite}
\usepackage{amsfonts,amsmath,amsthm}
\usepackage{algorithmic}
\usepackage{algorithm}
\usepackage{hyperref}

\newtheorem{lemma}{Lemma}
\newtheorem{theorem}{Theorem}
\newtheorem{definition}{Definition}
\newtheorem{proposition}{Proposition}

\newcommand{\opt}{\mathrm{opt}}
\newcommand{\non}{\mathrm{non}}
\newcommand{\same}{\mathrm{same}}
\newcommand{\high}{\mathrm{high}}
\newcommand{\set}{\mathrm{set}}
\newcommand{\bridge}{\mathrm{bridge}}

\begin{document}

\title{Novel  Analysis of  Population Scalability in Evolutionary Algorithms}
\author{Jun He\thanks{Corresponding author. \href{mailto:jun.he@ieee.org}{jun.he@ieee.org} }\\Department of Computer Science\\ Aberystwyth University, Aberystwyth, SY23 3DB, U.K. \\  
  \and Tianshi Chen \\Institute of Computing Technology\\ Chinese Academy of Sciences, Beijing 100190, China 
  \and Boris Mitavskiy\\Department of Computer Science\\ Aberystwyth University, Aberystwyth, SY23 3DB, U.K. }

\date{\today}
\maketitle

\begin{abstract}
Population-based evolutionary algorithms (EAs) have been widely applied to solve various optimization problems. The question of how the performance of a population-based EA  depends on the population size arises naturally. The performance of an EA may be evaluated by different measures, such as the average convergence rate to the optimal set per generation or the expected number of generations to encounter an optimal solution for the first time.  Population scalability is the performance ratio between  a benchmark EA and another  EA using identical genetic operators but a larger population size. Although intuitively the performance of an EA may improve if its population size increases, currently there exist only a few   case studies  for simple fitness functions. This paper aims at providing a general study for discrete optimisation. A novel approach  is introduced to analyse population scalability using the fundamental matrix.  The following two contributions summarize the major results of the current article. (1) We demonstrate rigorously that for elitist EAs with identical global mutation, using a lager population size always increases the average rate of convergence to the optimal set; and yet, sometimes, the expected number of generations needed to find an optimal solution (measured by either the maximal value or the average value) may increase, rather than decrease. (2) We establish sufficient and/or necessary conditions for the superlinear scalability, that is, when the average convergence rate of a $(\mu+\mu)$ EA (where $\mu\ge2$) is bigger than $\mu$ times that of a $(1+1)$ EA.
\end{abstract}

\section{Introduction}
Population-based evolutionary algorithms (EAs) have been widely applied to tackle a variety of optimization problems. A wide number of approaches is available to design efficient population-based EAs. Using a population delivers many benefits~\citep{prugel2010benefits}. A commonly accepted intuitive hunch is that the performance of such an EA may improve if its population size increases. Nonetheless, sometimes an intuitive rule of thumb may be bridgeable, hence a rigorous analysis is highly desirable. Currently there are only a few case studies for simple fitness functions, but no general result has been established so far.

Population scalability describes the relationship between the
performance of an  EA and its population size. The actual meaning of a population-based EA
should be interpreted as a family of EAs using identical genetic operators but
different population sizes. Considering a benchmark EA and another EA
  in the family with a population size  larger than the benchmark EA,   population
scalability (scalability for short) is measured intuitively as the ratio,
\begin{equation}\label{equScalability}
\mbox{scalability}= \frac{\mbox{performance of a benchmark EA}}{\mbox{performance of an  EA with a larger population size}}.
\end{equation}

To make use of the above formula, it is necessary to clarify the meaning of the ``performance'' of an EA.  Since EAs are iterative methods, the following two measures are fundamental in evaluating their performance from both, theoretical and practical points of view.
\begin{description}
\item  [Convergence rate:] the rate of an EA converging to the optimal set per generation~\citep{suzuki1995markov,he1999convergence}.
 The convergence rate is a measure applicable in both, numerical and discrete optimization problems.

\item  [Expected  number of generations:] the average number of generations needed to encounter an optimal solution for the first time \footnote{In EAs (but never in iterative methods), the  performance  is also evaluated by  the expected number of fitness evaluations needed to find an optimal solution. Since in   a $(\mu+\mu)$ EA, the  number of fitness evaluations    per generation is fixed to $\mu$, so the expected number of fitness evaluations   $=\mu \times$  the expected number of generations.  Therefore it is sufficient to study population scalability using the number of generations.}.
This measure  is  not suitable for numerical optimization, where an EA usually needs an infinite number of generations to find an exact solution.
\end{description}

To simplify the analysis, a $(1+1)$ EA plays the role of the benchmark EA in the paper. Another EA is a
$(\mu+\mu)$ EA  where  $\mu$ $(\ge 2)$ is the population size (an integer). Furthermore EAs under consideration satisfy the following conditions: (1) they are applied to tackle discrete optimization problems; (2) they are convergent; (3) genetic operators include mutation and selection only; and (4) genetic operators are selected in the same fashion at every generation.

The problem of population scalability is rather challenging in the theory of population-based EAs. In order to estimate population scalability, it is necessary to acquire the exact value of the  convergence rate or the expected number of generations it takes to reach an optimal solution for both, $(1+1)$ and $(\mu+\mu)$ EAs. The difficulty is that the convergence rate depends on the initial population and varies from one generation to another; and the expected number of generations depends on the initial populations. Furthermore the search spaces corresponding to the $(1+1)$ and $(\mu+\mu)$ EAs have different dimensions. Therefore the notion of population scalability should be based on the  ``overall''  performance of an EA, but how does one define the ``overall'' performance?

Convergent EAs can be modelled via absorbing Markov chains \citep{rudolph1998finite,he2003towards}, and one of the most elegant aspects of the theory of absorbing Markov chains revolves around  the notion of the fundamental matrix, $\{N(X, \, Y)\}_{X, \, Y \text{ are transient states}}$ where $N(X,Y)$ is the expected number of visits to the transient state $X$ starting with a transient states $Y$ prior to absorption. 

Clearly given a $(1+1)$ EA and a $(\mu+\mu)$ EA (where $\mu\ge 2$), it is impossible to make an ``entrywise'' comparison between their fundamental matrices  (the expected number of visits) since the dimensions of the corresponding matrices are different. Instead, such a comparison should be based on a non negative-valued function of the matrix, such as the spectral radius or a matrix norm of the fundamental matrix from the viewpoint of mathematics. The spectral radius of a matrix is the maximum of the absolute values of its eigenvalues.

This paper focuses on the spectral radius $\rho(\mathbf{N})$. In section~\ref{secMarkov} we show that  $1/\rho(\mathbf{N})$ is the average rate of convergence to the optimal set; and $\rho(\mathbf{N})$ is a "max-min" value  related to the expected number of visits to a transient state. Population scalability is measured rigorously in terms of the ratio of the spectral radii of the corresponding fundamental matrices:
\begin{align}
&\frac{\mbox{spectral radius of the fundamental matrix of  the $(1+1)$ EA}}{\mbox{spectral radius of the fundamental matrix of the $(\mu+\mu)$ EA}}.
\end{align}

The aim of the current article is to compare the average convergence rate towards an optimal solution (i.e. an absorbing set of states) of the corresponding Markov transition matrices modelling a $(\mu+\mu)$ EA and a $(1+1)$ EA. This makes the paper very different from the majority of prior research trends, that study the expected number of generations. This work addresses the following two fundamental questions.
\begin{enumerate}
\item Does the average convergence rate increase as the population size does, and, if so, then under what kind of circumstances does this take place?
Intuitively this seems trivial, since a $(\mu+\mu)$ EA employs more individuals than a $(1+1)$ EA does, nonetheless, a proof is required to confirm this.

\item As the population size increases from $1$ to $\mu$, under what kind of circumstances does the average convergence rate increase by a factor bigger than $\mu$?
This is another intuitive principle, since the number of individuals employed by a $(\mu+\mu)$ EA  is $\mu$ times that by the corresponding $(1+1)$ EA.
\end{enumerate}

The paper is organized as follows: A review of previous related research is given in Section \ref{secRelated}. Convergence rate and population scalability are formally introduced in Section~\ref{secMarkov}. Section~\ref{secScalability} aims at answering the first question stated above while Section~\ref{secGeneralStudy} aims at answering the second one. Sections~\ref{secCaseStudy1} and \ref{secCaseStudy2} are devoted to the case studies of non-bridgeable and bridgeable fitness landscapes.  Finally, Section~\ref{secConcluion} concludes the paper and   discuss other types of EAs.
\section{Related Work }
\label{secRelated}
The study of  population scalability in evolutionary computation can be traced to early 1990s. \citet{goldberg1992genetic} presented a population sizing equation to show how a large population size helps an EA to distinguish between good and bad building blocks on some test problems.  \citet{muhlenbein1993science} studied the critical (minimal) population size that can guarantee convergence to the optimum. An adoptive scheme to control the population size has been proposed in \citet{arabas1994gavaps} and the effectiveness of the proposed methodology has been validated through an empirical study. A review of various techniques to control EA's parameters where an adjustment of the population size has been emphasized as an important research issue appears in \citet{eiben1999parameter}. A link between the population size and the quality of the solution has been exhibited in \citet{harik1999gambler} via an analogy between one-dimensional random walks and EAs. While the approximate population sizing models proposed in the investigations mentioned above may shed some light on deciding a ``promising'' population size, the effectiveness of the models has been validated only via various case studies based upon specific optimization problems.

There do exist a few rigorous results about population scalability. As one of the earliest rigorous analysis,  \citet{he2002individual} investigated how the expected hitting time of EAs varies as the population size increases. A while later \citet{he2006analysis} exhibited a link between population scalability and parallelism. A study of the population scalability of the $(1 + \mu)$ EA on three pseudo-Boolean functions, Leading-Ones, One-Max and Suf-Samp appears in \citet{jansen2005choice}. \citet{lassig2011adaptive} presented a runtime analysis of a $(1+\mu)$ EA with an adaptive offspring size $\mu$ on several pseudo-Boolean functions. An analysis of how the running time of a $(\mu+1)$ EA on the Sphere function scales up with respect to the problem size $n$ appears in \citet{jagerskupper2005rigorous}. In \citet{jansen2005real} it has been shown that the running time of the $(\mu+1)$ EA with a crossover operator on the Real Royal Road function is polynomial on average, while that of an EA with mutation and selection only is exponentially large with an overwhelming probability. A rigorous run-time analysis of the $(\mu+1)$ EA on a specific pseudo-Boolean function has been carried out in \citet{witt2006runtime,witt2008population}.  A rigorous runtime analysis of selecting the population size with respect to the $(\mu+1)$ EA on several pseudo-Boolean functions appears in \citet{storch2008choice}. A runtime analysis of both $(1 + \mu)$ and $(\mu+1)$ EA on some instances of Vertex Covering Problems is provided in \citet{oliveto2009analysis}. A run-time analysis of $(\mu+1)$ EAs with diversity-preserving mechanisms on the Two-Max problem  has been implemented in \citet{friedrich2009analysis}. An upper bound on the number of generations it takes a $(\mu + \mu)$ EA to encounter an optimal solution for the first time on the two well-known unimodal problems, Leading-Ones and One-Max, has been obtained in \citep{chen2009new}. The effect of population size in evolutionary multi-objective optimization has been considered in \citet{giel2010effect}. It has been shown that only the population-based EA is successful, while all the other individual-based algorithms fail on a specified class of pseudo-Boolean functions.
Nonetheless, all of the available theoretical results are mainly restricted to several simple algorithms for tackling specific problems. In other words, the up-to-date knowledge is limited to case studies only \citep{oliveto2007time}.

In contrast with the previous investigations, the current paper aims at drawing general results that apply to all discrete optimisation problems. The study is based on the fundamental matrix of the Markov chain modelling an EA. Such an approach can be traced back to an early work on asymptotic convergence properties of EAs in \citet{fogel1994asymptotic} and it has been applied to analysing elitist EAs in \citet{he2003towards}.

\section{Evolutionary Algorithms, Absorbing Markov Chains and Population Scalability}
\label{secMarkov}
\subsection{Formalization of   EAs}
Without loss of generality, consider the problem of maximizing  a fitness function $f(x)$.
$$
  \max   f(x), \qquad  {x \in  D}, \quad \mbox{subject to  constraints},
$$
where $f(x)$ is a fitness function, and $D$ is its domain (a finite set). For instance, $D$ is the set of all Boolean formulas in the satisfiability problem \citep{zhou2009comparative}, or the set of all possible vertex covers in the vertex cover problem \citep{oliveto2009analysis}.
Arrange the values of the fitness function in the order from high to low, that are called \emph{fitness levels}.

To alleviate the complexity of theoretical analysis, suppose that all constraints in the above problem have been removed through a constraint handling method. Under this circumstance,  all solutions in $D$ are thought to be feasible. Practical and theoretical analysis of constraint handling in evolutionary computation may be found in \citep{coello2002theoretical,zhou2007runtime}, for instance.

We consider an EA that makes use of an extra archive for keeping the best found solution. The archive itself is not involved in generating a new population. The general design of a $(\mu+\mu)$  EA  with an archive is described in Algorithm~\ref{alg1}. In the description, $t$ denotes the generation counter. $\Phi_t=(\phi_{t,1}, \cdots, \phi_{t,\mu})$
represents the population at the $t^\text{th}$ generation, where $\phi_{t,1}, \cdots, \phi_{t,\mu}$ are $\mu$ individuals.  $\Phi_{t+1/2}$
denotes the  offspring population generated via mutation from the population $\Phi_t$.

 \begin{algorithm}
\caption{A $(\mu+\mu)$   EA with an Archive (where $\mu$ is the population size)} \label{alg1}
\begin{algorithmic}[1]
\STATE \textbf{input}: fitness function;
\STATE generation counter $ t\leftarrow 0$;
\STATE initialize $ \Phi_0$;
\STATE an archive   keeps the best solution in $\Phi_0$;
\WHILE{(no optimal solution is found)}
\STATE $\Phi_{t+1/2}\leftarrow$ each individual in $\Phi_t$ generates a  child by mutation;
\STATE evaluate the fitness of each individual in $\Phi_{t+1/2}$;
\STATE $\Phi_{t+1}\leftarrow$   selected from $ \Phi_t, \Phi_{t+1/2}$;
\STATE  update the archive  if the best solution in $\Phi_{t+1}$ is better than it;
\STATE $t\leftarrow t+1$;
\ENDWHILE \STATE \textbf{output}:  the maximum of the fitness function.
\end{algorithmic}
\end{algorithm}

The initial population  is selected at random in such a way that any possible population may be chosen with a positive probability. Rather general definitions of mutation and selection operators appear below.
 \begin{itemize}
\item A  \emph{mutation  operator} is represented via an $\mathcal{S}^{(1)}$ by $\mathcal{S}^{(1)}$ Markov transition probability matrix the entries of which are given as
$$P_M(x,y)=P(\phi_{t+1/2}=y \mid \phi_t =x), \quad x, y \in \mathcal{S}^{(1)}.$$
Here $P_M(x,y)$ denotes the probability of going from  $x$ to $y$. $\phi_t$  is an individual in  the $t^{\mathrm{th}}$ generation population and $\phi_{t+1/2}$ the child of $\phi_t$ after mutation. 
 $\mathcal{S}^{(1)} =D$ is called the \emph{space of individuals}. Individuals $\phi_t$ and $\phi_{t+1/2}$ represent random variables, while $x$ and $y$ denote their states. 

\item  A  \emph{selection operator} is represented via an $\mathcal{S}^{(\mu)} \times \mathcal{S}^{(\mu)}$ by $\mathcal{S}^{(\mu)}$ probability transition matrix, the entries of which are introduced below:
\begin{align*}
 P_S(X, Y; Z) = P(\Phi_{t+1}=Z \mid \Phi_t=X, \Phi_{t+1/2}=Y), \quad X, Y , Z \in  \mathcal{S}^{(\mu)}.
\end{align*}
Here  $P_S(X, Y; Z)$ represents the probability of selecting $\mu$ individuals from populations $X$  and  $Y$ (children of   $X$) and then forming the next parent population $Z$. $\Phi_t$  is the $t^{\mathrm{th}}$ generation population, $\Phi_{t+1/2}$ the population of children of $\Phi_t$ after mutation, and $\Phi_{t+1}$  is the $(t+1)^{\mathrm{th}}$ generation population. 
$\mathcal{S}^{(\mu)}$ is the Cartesian product $\prod^{\mu}_{i=1} \mathcal{S}^{(1)}$, called a \emph{space of populations} if $\mu \ge 2$. Populations $\Phi_t$, $\Phi_{t+1/2}$ and $\Phi_{t+1}$ are random variables, while  $ X,  Y, Z$ denotes their states in the space of populations.    Superscripts $^{(1)}$ and $^{(\mu)}$ are used to distinguish between the space of individuals and space of populations.

A natural requirement on the selection operator is that all the individuals in $Z$ must come from these in $X$ or in $Y$. If  $Z$  contains an individual that is neither an individual of $X$ nor of $Y$, then  the  probability of going from  $X$ and $Y$ to  $Z$ is $0$.
\end{itemize}

The stopping criterion is that the algorithm will terminate once an optimal solution is found. This criterion is assumed only for the sake of convenience of the theoretical analysis. Apparently, we can simply ignore what happens after an EA encounters an optimal solution. EAs considered in the paper are convergent. Starting from any initial population, an EA can find an optimal solution after a finite number of generations.

The mathematical framework introduced above incorporates a wide class of  EAs as it doesn't assume any implementation details.
Table~\ref{notation} summarizes  the special notation appearing in the paper.

\begin{table*}[ht]
\begin{tabular*}{1\textwidth}{ r  p{0.75\textwidth}  }
\hline
 $^{(1)}$ and $^{(\mu)}$& distinguish between a $(1+1)$ EA and a $(\mu+\mu)$ EA when necessary\\
 $\mathcal{S}^{(1)} $ & the  set   of individuals \\
 $\mathcal{S}^{(\mu)}$ &   the set of populations of size $\mu$, $=$ Cartesian product $\prod^{\mu}_{i=1} \mathcal{S}^{(1)}$ \\
 $x,y,z  $ &  individuals, also called states in $\mathcal{S}^{(1)}$  \\
  $X,  Y,Z  $&    populations, also called states in $\mathcal{S}^{(\mu)}$     \\
   $\mathcal{S}^{(1)}_{\opt} $ & the set of  individuals that are optimal\\
  $\mathcal{S}^{(1)}_{\non}  $   &  the set of individuals that are not optimal  \\
 $\mathcal{S}^{(\mu)}_{\mathrm{opt}} $ & the set of  populations that contain at least one optimal individual \\
  $\mathcal{S}^{(\mu)}_{\mathrm{\non}}$ & the set of populations that contain no optimal individual \\
  $\mathcal{S}^{(\mu)}_{\same}(x)$ & the set of populations the best individual of which is $x$\\
  $\mathcal{S}^{(\mu)}_{\high}(x)$& the set of populations with best individual's fitness $>f(x)$\\
  $\mathcal{S}^{(\mu)}_{\bridge}(x)$& the set of populations that contain $x$'s bridgeable point\\
  $P(X, \mathcal{S}^{(\mu)}_{\set}(x))$ & the probability of going from $X$ to the set $\mathcal{S}^{(\mu)}_{\set}(x)$\\
$x_{\rho}$ & $=\arg\max P(x,x)$ among all non-optimal states $x \in \mathcal{S}^{(1)}_{\non}$ \\
  $\Phi_t$  & the population at the $t^{\text{th}}$ generation \\
  $\Phi_{t+1/2}$  & the offspring population of $\Phi_t$ after mutation \\
 $q_t(X)$ & the probability that  $\Phi_t=X$ \\
 $\mathbf{q}_t$ & the vector to represent the  probabilities of $\Phi_t$ in all non-optimal states\\
 $m(X)$ & the expected number of generations needed to find an optimal solution when starting from $X$ \\
  $\mathbf{Q}_{x,x}$ & the transition probability submatrix within the set $\mathcal{S}^{(\mu)}_{\same}(x)$\\
\hline
\end{tabular*}
\caption{Notation Table}
\label{notation}
\end{table*}
\subsection{Absorbing Markov Chains and their Fundamental Matrices}
\label{sec-markov}
The discussion in the subsection follows the general analytic framework of the absorbing Markov chains for analysing EAs  initiated in \citet{he2003towards}.
According to the stopping criterion, an EA halts once an optimal solution is found. So if $\Phi_t = X$ is an optimal state, we let
$\Phi_{t+1}= \Phi_{t+2} =\cdots$ $= X$   for all the future states.
Thus the sequence $\{\Phi_t, t=0, 1,\cdots\}$ can be modelled by an absorbing Markov chain \footnote{An absorbing Markov chain is a Markov chain where starting from every state, the chain can reach an absorbing state.  An absorbing state is a state from which it is impossible to leave\citep[p.416]{grinstead1997introduction}.}.
Let $\mathbf{P}$ be its transition matrix, having   entries
$$P(X,Y)=P(\Phi_{t+1}= Y \mid  \Phi_t = X), \quad X,Y \in \mathcal{S}.$$
Individuals or populations are called states when we speak about the corresponding Markov chains. For the Markov chains modelling convergent EAs, an optimal state is always an absorbing state while a non-optimal state is always a transient state.

The probability of going from $X$ to a  set $\mathcal{S}_{\set}$ is denoted by
\begin{align*}
P(X,\mathcal{S}_{\set})&=P(\Phi_{t+1} \in \mathcal{S}_{\set} \mid  \Phi_t = X).
\end{align*}

The transition matrix of an absorbing Markov chain $\mathbf{P}$ can be written in the following canonical form.
\begin{equation}
\label{equTransitionMatrix} \mathbf{P} =  \begin{pmatrix}
  \mathbf{I} & \mathbf{{O}} \\
   *  & \mathbf{Q}
 \end{pmatrix},
\end{equation}
where $\mathbf{I}$ is the identity  matrix indicating transitions within the optimal set, and $\mathbf{O}$ denotes the zero matrix (apparently representing transitions from optimal to non-optimal states, as discussed above). The matrix $ \mathbf{Q}$ denotes transitions within non-optimal  states.
The part $*$ represents transitions from the non-optimal states to the optimal ones.

 Perhaps the most elegant part of the theory of absorbing Markov chains revolves around the notion of a fundamental matrix.
\begin{definition}\citep[Definition 11.3]{grinstead1997introduction}
  The matrix $\mathbf{N}=(\mathbf{I}-\mathbf{Q})^{-1}$ is called the \emph{fundamental matrix}, where its entry  $N(X,Y)$ gives the expected number of visits to a transient state $X$ starting from a transient state $Y$  before being absorbed.
\end{definition}

The expected number of visits  has a direct link to the
the expected number of generations needed to encounter an optimal solution for the first time.

\begin{lemma}\label{lemFundMatrix}
\citep[Theorem 11.5]{grinstead1997introduction}
Let $m(X)$ denote the expected number of generations needed to encounter an optimal solution for the first time when starting from a transient state $X$.  Then
\begin{equation}
m(X) = \sum_{Y \in \mathcal{S}_{\non}} N(X,Y).
\end{equation}
\end{lemma}

In the paper we consider two special values of the expected number of generations it takes to reach an optimum solution.
The first one is the maximum value of the expected number of generations, given as
\begin{align}
\max \{ m(X); X \in \mathcal{S} \} =\max \{\sum_{Y \in \mathcal{S}_{\non}} N(X,Y); X \in \mathcal{S} \} .
\end{align}
The above value is the supreme of the expected number of generations it takes to reach an optimal solution among all possible initializations.

The second one is the average value of the expected number of generations it takes to reach an optimum over all of the transient states, given as
\begin{align}
\frac{1}{\mid \mathcal{S}_{\non}\mid} \sum_{X \in \mathcal{S}_{\non}} m(X)=\frac{1}{\mid \mathcal{S}_{\non}\mid} \sum_{X,Y \in \mathcal{S}_{\non}} N(X,Y),
\end{align}
where $\mid \mathcal{S}_{\non}\mid$ denotes the cardinality of the set of transient states. The above value is achieved when the initial population is chosen uniformly at random from the set of all transient states. Analogously, an alternative notion of the expected number of generations it takes to encounter an optimal individual for the first time when any population (not necessarily one containing no optimal individual) is selected uniformly at random is given as follows:
\begin{align}
\frac{1}{\mid \mathcal{S} \mid} \sum_{X \in \mathcal{S}} m(X).
\end{align}

\subsection{Average Convergence Rate}
In this subsection, we introduce the notion of the average convergence rate. The convergence rate of an EA measures how fast  $\Phi_t$ converges to the optimal set per generation. In EAs, it is formally defined by the conditional probability of an EA converging to the optimal set, that is
\begin{align*}
 P(\Phi_{t} \in \mathcal{S}_{\opt} \mid \Phi_{t-1} \in \mathcal{S}_{\non})
=  1-  P(\Phi_{t} \in \mathcal{S}_{\non} \mid \Phi_{t-1} \in \mathcal{S}_{\non}).
\end{align*}
We denote the probability that $\Phi_t$ is at a transient state $X$  by
$$
q_t(X)=P(\Phi_t=X).
$$
Write all transient states in a vector form~\footnote{$\mathbf{v}$ represents a column vector and the transpose notation, $\mathbf{v}^T$, allows to represent this column vector in the form of a row vector. The vector $\mathbf{v}$ is also denoted by $[v(X)]$ and its entry $v(X)$ by $[\mathbf{v}]_X$.}: $(X_1, X_2, \cdots   )^T$.   Then  the vector
$$\mathbf{q}_t= (q_t(X_1), q_t(X_2), \cdots )^T $$ denotes the probabilities of  $\Phi_t$ in   all transient states.
The corresponding Markov chain is then represented by a matrix iteration:  
\begin{align}
\mathbf{q}^T_{t+1}=  \mathbf{q}^T_t \mathbf{Q},  \mbox{  or  }\mathbf{q}_{t+1}=   \mathbf{Q}^T \mathbf{q}_t.
\end{align}

Since the EA is initialized at random so that every possible population is selected with a positive probability (for instance, uniformly at random), each state can be chosen as an initial state  with a positive probability. This means  that $q_0(X) > 0$ for all $X \in \mathcal{S}_{\non}$.  We denote this by $\mathbf{q}_0 >0$. 

Write the probability of $\Phi_t$ being in the transient set in the $1$-norm for~\footnote{For a vector $\mathbf{v}$,
$\parallel \mathbf{v}  \parallel_1=\sum_{i=1}^n \mid v_i\mid $. For a square matrix $\mathbf{A}=[A_{ij}]$, $\parallel \mathbf{A} \parallel_1$= $\max_{1 \le j \le n} \sum^n_{i=1} |A_{ij} | $.  }.
$$ \parallel \mathbf{q}_t  \parallel_1 = P(  \Phi_{t} \in \mathcal{S}_{\non}).
$$ 

Since the EA is initialized at random, $\parallel \mathbf{q}_t  \parallel_1 = P(  \Phi_{t} \in \mathcal{S}_{\non})>0$. 

Since $P(\Phi_{t-1} \in \mathcal{S}_{\opt},\Phi_{t} \in \mathcal{S}_{\non})=0$,
the conditional probability of staying within the non-optimal set in the $t^{\mathrm{th}}$ generation is
\begin{align*}
&  P(\Phi_{t} \in \mathcal{S}_{\non} \mid \Phi_{t-1} \in \mathcal{S}_{\non})\\& = \frac{P(\Phi_{t-1} \in \mathcal{S}_{\non},  \Phi_{t} \in \mathcal{S}_{\non})}{P(\Phi_{t-1} \in \mathcal{S}_{\non})}\\ &= \frac{P(\Phi_{t-1} \in \mathcal{S}_{\opt},\Phi_{t} \in \mathcal{S}_{\non})+P(\Phi_{t-1} \in \mathcal{S}_{\non},\Phi_{t} \in \mathcal{S}_{\non})}{P(\Phi_{t-1} \in \mathcal{S}_{\non})}\\
&= \frac{P( \Phi_{t} \in \mathcal{S}_{\non})}{P(\Phi_{t-1} \in \mathcal{S}_{\non})} \\
&= \frac{\parallel \mathbf{q}_{t} \parallel_1}{\parallel \mathbf{q}_{t-1} \parallel_1} .
\end{align*}

Thus, the geometric mean of the conditional probabilities of staying within the non-optimal set for $t$ generations is
\begin{align}
  \left(\prod^{t}_{s=1} P(\Phi_{s} \in \mathcal{S}_{\non} \mid \Phi_{s-1} \in \mathcal{S}_{\non} )\right)^{1/t} = \left(   \frac{\parallel \mathbf{q}_{t} \parallel_1}{\parallel \mathbf{q}_{0} \parallel_1} \right)^{1/t}. \label{eqnGeometicMean}
\end{align} 

Next we define the average convergence rate for $t$ generations based on the geometric mean above.
\begin{definition}
The \emph{average rate of  convergence to the optimal set for $t$ generations} is
\begin{align*}
  1- \left(\prod^{t}_{s=1} P(\Phi_{s} \in \mathcal{S}_{\non} \mid \Phi_{s-1} \in \mathcal{S}_{\non} )\right)^{1/t} .
\end{align*}  
\end{definition} 
Now we consider the limit of the rate as $t$ increases towards $+\infty$. 

\begin{lemma}\label{rateLemmaIllustrate}
The average rate of convergence to the optimal set for $t$ generations satisfies
\begin{align*}
1-\lim_{t \to +\infty}  \left(\prod^{t}_{s=1} P(\Phi_{s} \in \mathcal{S}_{\non} \mid \Phi_{s-1} \in \mathcal{S}_{\non} )\right)^{1/t}=
1-\rho(\mathbf{Q}) .
\end{align*}
\end{lemma}
\begin{proof}
(1) From the matrix iteration
  $\mathbf{q}_t= \mathbf{Q}^T   \mathbf{q}_{t-1} $, we get
\begin{align*}
\parallel \mathbf{q}_t \parallel_1 \le  \parallel (\mathbf{Q}^T)^t \parallel_1 \parallel \mathbf{q}_0 \parallel_1.
\end{align*}

Since  the EA is initialized at random,  and $\mathbf{q}_0 > \mathbf{0}$,  the average  rate of convergence to the optimal set after $t$ iterations satisfies
\begin{align*}
1-\left(\frac{\parallel \mathbf{q}_t \parallel_1}{\parallel \mathbf{q}_0 \parallel_1} \right)^{1/t} \ge 1- \parallel (\mathbf{Q}^T)^t \parallel_1^{1/t}.
\end{align*}
According to the Gelfand's spectral radius formula~\footnote{Gelfand's spectral radius formula says that for any induced matrix norm $\parallel \mathbf{A} \parallel$, its spectral radius $\rho(\mathbf{A})=\lim_{t \to \infty}\parallel \mathbf{A}^t \parallel^{1/t}$   \citep[p.619]{meyer2000matrix}.}, as $t \to +\infty$,
\begin{align}
\label{equLowerBound}
1-\lim_{t \to +\infty} \left(\frac{\parallel \mathbf{q}_t \parallel_1}{\parallel \mathbf{q}_0 \parallel_1} \right)^{1/t} \ge 1-\lim_{t \to +\infty} \parallel (\mathbf{Q}^T)^t \parallel_1^{1/t}=1-\rho(\mathbf{Q}^T)=1-\rho(\mathbf{Q}).
\end{align}

(2) Since  $\mathbf{Q}\ge 0$, according to the Perron-Frobenius theorems \citep[p.670]{meyer2000matrix}, $\rho(\mathbf{Q})$ is an eigenvalue of $\mathbf{Q}$ (and also $\mathbf{Q}^T$) and the corresponding eigenvector $\mathbf{v}\ge 0$. In particular,
\begin{align*}
\rho(\mathbf{Q}) \mathbf{v} =  \mathbf{Q}^T \mathbf{v}.
\end{align*}

Let $\min \mathbf{q}_0$ denote the minimum value of all the entries in the vector $\mathbf{q}_0$. Since $\mathbf{q}_0 >0$, $\min \mathbf{q}_0 >0$ as well.  Normalize the vector $\mathbf{v}$ so that~\footnote{For a vector $\mathbf{v}$,
$\parallel \mathbf{v}  \parallel_{\infty}=\max_{i=1, \cdots,n} \mid v_i\mid $.} $\parallel \mathbf{v} \parallel_{\infty} =\min \mathbf{q}_0 $. We split $  \mathbf{q}_0$ into two parts,
\begin{align*}
\mathbf{q}_0 =\mathbf{v}+\mathbf{w},
\end{align*}
where
$\mathbf{w}\ge 0$.
Thus, since $\mathbf{w}\ge 0$ and $\mathbf{Q} \ge 0$, we deduce that
\begin{align*}
  \mathbf{q}_t =\mathbf{Q}^T \mathbf{q}_{t-1} =  (\mathbf{Q}^T)^t \mathbf{q}_0 = (\mathbf{Q}^T)^t (\mathbf{v}+\mathbf{w}) \ge  (\mathbf{Q}^T)^t  \mathbf{v}=(\rho(\mathbf{Q}))^t \mathbf{v}.
\end{align*}
It follows then that
\begin{align*}
\left(\frac{\parallel \mathbf{q}_t \parallel_1}{\parallel \mathbf{q}_0 \parallel_1} \right)^{1/t}  \ge \rho(\mathbf{Q}) \left(\frac{\parallel \mathbf{v}  \parallel_1}{\parallel \mathbf{q}_0 \parallel_1} \right)^{1/t} \to \rho(\mathbf{Q})
\end{align*}
as $t \rightarrow + \infty$. The inequality is equivalent to the one below:
\begin{align*}
1- \lim_{t \to +\infty }\left(\frac{\parallel \mathbf{q}_t \parallel_1}{\parallel \mathbf{q}_0 \parallel_1} \right)^{1/t} \le 1- \rho(\mathbf{Q}).
\end{align*}

The desired conclusion follows by combing the inequality above with   Inequality~(\ref{equLowerBound}) and Equality (\ref{eqnGeometicMean}).
\end{proof}

\begin{definition}
The \emph{average rate of convergence to the optimal set} is
\begin{align*}
 1-\lim_{t \to +\infty}  \left(\prod^{t}_{s=1} P(\Phi_{s} \in \mathcal{S}_{\non} \mid \Phi_{s-1} \in \mathcal{S}_{\non} )\right)^{1/t} = 1-\rho(\mathbf{Q}) .
\end{align*}
\end{definition}
The ``average'' is the geometric mean that is taken over all generations under the condition of randomized initialization~\footnote{The rate is similar to another average rate of convergence   based on the logarithmic mean~\citep[p.73]{varga2009matrix},   $-\ln \rho(\mathbf{Q})$. The difference between the two notions of rate is not particularly significant~\citep{he2012pure}} .

\subsection{The Spectral Radius  of the Fundamental Matrix}
Given a $(1+1)$ EA and a $(\mu+\mu)$ EA (where $\mu\ge 2$), it is impossible to make an ``entrywise'' comparison of  their fundamental matrices (the expected number of visits) since their dimensions are different.  Instead, the comparison  should be based on a  measure of all entries.  The spectral radius  or matrix norms of the fundamental matrix thus play the role of such a measure.

In this subsection we discuss the spectral radius of the fundamental matrix. First we show that $1/\rho(\mathbf{N})$ equals the average convergence rate. This can be seen from the following lemma.

\begin{lemma}
\label{lemRadiiRelation}
The spectral radii of the transition probability submatrix  $\mathbf{Q}$ and the fundamental matrix $\mathbf{N}$ are related as follows:
\begin{equation}
 \rho(\mathbf{N}) = ({1-\rho(\mathbf{Q})})^{-1}.
\end{equation}
\end{lemma}

\begin{proof}
From the definition of the fundamental matrix, it follows that $\lambda$ is an eigenvalue of $\mathbf{Q}$ if and only if $(1-\lambda)^{-1}$ is an eigenvalue of $\mathbf{N}$.

Since $\mathbf{Q}$ is non-negative, according to Perron-Frobenius Theorems \citep[p.670]{meyer2000matrix}, $\rho(\mathbf{Q})$ is an eigenvalue of $\mathbf{Q}$ such that
$$
\rho(\mathbf{Q}) \ge \mid \lambda\mid,
$$
where $\lambda$ is any eigenvalue of $\mathbf{Q}$.

On the other hand, $({1-\rho(\mathbf{Q})})^{-1}$ is an eigenvalue of $\mathbf{N}$ and satisfies
\begin{align*}
&\frac{1}{1-\rho(\mathbf{Q})} \ge \frac{1}{1-\mid \lambda \mid } \ge \frac{1}{\mid 1-\lambda \mid},&
\end{align*}
so that $({1-\rho(\mathbf{Q})})^{-1}$ is the spectral radius of $\mathbf{N}$.
\end{proof}

Next the lemma below  shows that $\rho(\mathbf{N})$ is a ``max-min'' value related to the expected number of visits to   transient states. 
\begin{lemma}
\label{lemMaxMin}
The  spectral radius of the fundamental matrix equals 
\begin{align}
\rho(\mathbf{N})&=\max_{\mathbf{q}_0 }\,\,\,\min_{Y: q_0(Y) \neq 0}\frac{ \sum_{X \in \mathcal{S}_{\non}} N(Y,X)q_0(X)}{q_0(Y) },
\end{align}where  $q_0(X)=P(\Phi_0=X)$ is the probability that the initial population $\Phi_0$ is at the state $X$ and $N(Y,X)$ is the expected number of visits to the state $X$ starting from the state $Y$. 
\end{lemma}

\begin{proof}
The lemma is a direct  application of the Collatz-Wielandt formula~\footnote{The "max-min" version of the Collatz-Wielandt formula claims that for a nonnegative square matrix $\mathbf{A}=[A_{ij}]$, its spectral radius   $\rho(\mathbf{A})=\max_{\mathbf{x} \in \mathcal{N}}{g(\mathbf{x})} $, where $g(\mathbf{x})=\min_{1\le i \le n, x_i \neq 0} \frac{[\mathbf{A x}]_i}{[\mathbf{x}]_i}$ and the set $\mathcal{N}=\{ \mathbf{x}; \mathbf{x} \ge \mathbf{0} \mbox{ with } \mathbf{x} \neq \mathbf{0} \}$
\citep[p670]{meyer2000matrix}. There exists a "min-max"  version of the Collatz-Wielandt formula, which is applicable   too.}.
\begin{align*}
\rho(\mathbf{N})=\rho(\mathbf{N}^T)&=\max_{\mathbf{q}_0 }\,\,\, \min_{Y: q_0(Y) \neq 0}\frac{ [\mathbf{N}^T \mathbf{q}_0]_Y}{[  \mathbf{q}_0]_Y}\\
&=\max_{\mathbf{q}_0 }\,\,\,\min_{Y: q_0(Y) \neq 0}\frac{ \sum_{X \in \mathcal{S}_{\non}} N(Y,X)q_0(X)}{q_0(Y) },
\end{align*}
which is the conclusion.
\end{proof}

Eventually, the lemma below establishes lower and upper bounds on the spectral radius of the fundamental matrix.

\begin{lemma}
\label{lemSpectralBoundsFund}
The  spectral radius of the fundamental matrix satisfies the following inequality:
\begin{align}
\min_{X \in \mathcal{S}_{\non} } m(X) \le \rho(\mathbf{N}) \le   \max_{X \in \mathcal{S}_{\non} }  m(X).
\end{align}
\end{lemma}

\begin{proof}
The lemma is a direct consequence of the following fact~\footnote{The fact is given  in \citep[Exercise 8.2.7]{meyer2000matrix}. The result in Exercise 8.2.7   is stated only for positive matrices, yet an identical argument that replaces the Collatz-Wielandt formula  for positive matrices by the Collatz-Wielandt formula for non-negative matrices shows that the same fact holds for all non-negative matrices.}: given any $n \times n$ non-negative $\mathbf{A}=[a_{ij}]$, its spectral radius satisfies the inequalities
\begin{align*}
 \min_{i} \sum^n_{j=1} a_{ij} \le \rho( \mathbf{A}) \le  \max_{i} \sum^n_{j=1} a_{i,j}.
\end{align*}
Indeed, substituting $\mathbf{N}$ in place of $\mathbf{A}$ yields the desired conclusion.
\end{proof}

\subsection{Matrix Norms of the Fundamental Matrix}
Matrix norms \footnote{For a square matrix $\mathbf{A}=[A_{ij}]$,  $ \parallel \mathbf{A} \parallel_\infty = \max \limits _{1 \leq i \leq n} \sum _{j=1} ^n | A_{ij} |$ and $\parallel \mathbf{A} \parallel_a=(\sum_{i,j=1}^n |A_{ij}|) /n$.} can be used as a  measure of the expected number of visits. The $\infty$-norm of the fundamental matrix is given as
\begin{align}
\parallel \mathbf{N} \parallel_{\infty} = \max_{X \in \mathcal{S}_{\non}} \sum_{Y \in \mathcal{S}_{\non}} N(X,Y).
\end{align}

\begin{lemma}
 The $\infty$-norm of the fundamental matrix equals
$
\parallel \mathbf{N} \parallel_{\infty} = \max_{X \in \mathcal{S}_{\non}} m(X).
$
\end{lemma}

\begin{proof}
This follows immediately from the definition of the matrix $\infty$-norm and Lemma~\ref{lemFundMatrix}.
\end{proof}

The definition and the lemma above provide us with two equivalent interpretations of $\parallel \mathbf{N} \parallel_\infty$.
  $\parallel \mathbf{N} \parallel_\infty$ is the maximal value of the expected number of visits to the set of transient states among all possible initializations.
This is equivalent to saying that $\parallel \mathbf{N} \parallel_\infty$ is the maximal value of the expected number of generations to reach the optimal set among all possible starting transient states.

The   $a$-norm of the fundamental matrix is defined as
\begin{align}
\parallel \mathbf{N} \parallel_a =  \frac{1}{\mid \mathcal{S}_{\non} \mid } \sum_{X,Y \in \mathcal{S}_{\non}} N(X,Y).
\end{align}

\begin{lemma}
 The $\infty$-norm of the fundamental matrix is alternatively described as follows:
\begin{align}
\parallel \mathbf{N} \parallel_a = \frac{1}{\mid \mathcal{S}_{\non} \mid } \sum_{X  \in \mathcal{S}_{\non}} m(X).
\end{align}
\end{lemma}

\begin{proof}
This is an immediate consequence of the definition of the matrix $a$-norm and lemma~\ref{lemFundMatrix}.
\end{proof}

The definition and the lemma above reveal the following two equivalent meanings of $\parallel \mathbf{N} \parallel_a$.
  $\parallel \mathbf{N} \parallel_a$ is the average value of the expected number of visits to the set of transient states among all possible initial transient states.
  $\parallel \mathbf{N} \parallel_a$ is the average value of the expected number of generations it takes to reach the optimal set among all possible initial transient states.

\subsection{Population Scalability}
Given a $(1+1)$  EA and a $(\mu+\mu)$   EA (where $\mu \ge 2$) that exploit an identical mutation operator to optimize the same fitness function, population scalability is measured by the ratio between their performances.
As discussed in previous sections, there are different approaches to evaluate the   performance of an EA and, hence, there are several ways to measure population scalability.
\begin{definition}
\emph{Population scalability under the spectral radius}  of the fundamental matrix is
  \begin{align}
& \mbox{\emph{$\rho$-scalability}} (\mu) = \frac{\rho(\mathbf{N}^{(1)})}{\rho(\mathbf{N}^{(\mu)})}= \frac{1-\rho(\mathbf{Q}^{(\mu)})}{1-\rho(\mathbf{Q}^{(1)})} \\
    & \qquad = \frac{\mbox{\emph{average convergence rate of the $(\mu+\mu)$ EA}}}{\mbox{\emph{average convergence rate of the $(1+1)$ EA}}}.
\end{align}
\end{definition}
\begin{definition}
\emph{Population scalability   under the  $\infty$-norm} of the fundamental matrix is
  \begin{align}
   &\mbox{\emph{$\infty$-scalability}} (\mu)  = \frac{\parallel \mathbf{N}^{(1)} \parallel_{\infty}}{\parallel \mathbf{N}^{(\mu)} \parallel_{\infty}} \\
    &\qquad  = \frac{\mbox{\emph{maximum value of expected numbers of generations of the $(1+1)$ EA}}}{\mbox{\emph{maximum value of expected numbers of generations of the $(\mu+\mu)$ EA}}}.
\end{align}
\end{definition}
\begin{definition}
\emph{Population scalability under the $a$-norm} of the fundamental matrix is
  \begin{align}
   &\mbox{\emph{$a$-scalability}} (\mu)  = \frac{\parallel \mathbf{N}^{(1)} \parallel_{a}}{\parallel \mathbf{N}^{(\mu)} \parallel_{a}} \\
    & \qquad = \frac{\mbox{\emph{average value of expected numbers of generations of  the $(1+1)$ EA}}}{\mbox{\emph{average value of expected numbers of generations of   the $(\mu+\mu)$ EA}}},
\end{align}
where the average is taken over all of the transient states, excluding the absorbing state(s).

If considering the average over all the states, an alternative definition is given by
\begin{align}
   &\mbox{\emph{$\hat{a}$-scalability}} (\mu)  = \frac{\mid \mathcal{S}^{(1)}_{\non} \mid  }{\mid \mathcal{S}^{(\mu)}_{\non} \mid  }\frac{\parallel \mathbf{N}^{(1)} \parallel_{a}}{\parallel \mathbf{N}^{(\mu)} \parallel_{a}}.
\end{align}
\end{definition}

An essential part of the definitions above is that both EAs must adopt
identical mutation operators. This ensures that the comparison is meaningful.
Nonetheless, it is impossible for the selection operators to be identical.
Indeed even if the selection operators are of the same type, for example roulette wheel selection, the conditional probabilities determining the actual selection operators are never
identical under distinct population sizes.

The following  questions are fundamental when studying population scalability.
\begin{enumerate}
\item As the population size increases from 1 to $\mu$ (where $\mu  \ge 2$), is the
$
\mbox{\emph{scalability}}(\mu)>1?
$

If the population scalability is not bigger than $1$, then we say that the $(\mu+\mu)$  EA has \emph{no scalability} with respect to the  $(1+1)$ EA.

\item As the population size increases from $1$ to $\mu$ (where $\mu  \ge 2$), is the
$
\mbox{\emph{scalability}}(\mu)>\mu?
$

If the population scalability is greater than $\mu$, then we say that the $(\mu+\mu)$  EA has \emph{superlinear  scalability} with respect to the  $(1+1)$ EA.
\end{enumerate}

Population scalability is different from the relationship between the performance of an EA and its population size discussed in previous references such as~\citep{jansen2005choice}.  There the comparison of the two EAs is carried out in terms of the big O notation. The difference is clearly demonstrated through the following  question:
\begin{align*}
    \frac{\mbox{ {maximum value of expected number of generations for the $(1+1)$ EA}}}{\mbox{ {maximum value of expected numbers of generations for the $(\mu+\mu)$ EA}}}< 1?
\end{align*}

Within the  framework in \citet{jansen2005choice}, the question may be reformulated as
\begin{align*}
    \frac{\mbox{ {maximum value of expected numbers of generations of $(1+1)$ EA}}}{\mbox{ {maximum value of expected numbers of generations of $(\mu+\mu)$ EA}}}=O(1)?
\end{align*}
Here $O(1)$ is big O notation. Nonetheless, there is a drawback in using the big O notation when studying population scalability. For example, $O(1)$ does not distinguish between the case when the expected number of generations the $(1+1)$ EA takes to reach an optimum for the first time is $100$ times that the  $(\mu+\mu)$ EA takes, and the case when the expected number of generations the $(1+1)$ EA takes is $1/100$ times that the $(\mu+\mu)$ EA takes. In this sense, population scalability analysis is  different from the work in \citet{jansen2005choice}.

The notion of population scalability is similar to that of the speedup widely used when analysing parallel algorithms. Nonetheless, population scalability doesn't depend on the number of parallel computing processors. There is a link between  superlinear population scalability and superlinear speedup in parallel EAs. If each individual is assigned to a processor, then EAs turn into parallel EAs. Under this circumstance,  superlinear scalability implies  superlinear speedup  if ignoring the communication cost. An interesting question  in parallel EAs is when and how   superlinear speedup  phenomenon happens~\citep{andre1998parallel,alba2002parallelism,alba2002parallel}.

There is an essential difference between the notions of population scalability and that of No Free Lunch Theorems~\citep{wolpert1997no}. Population scalability compares the performance of   two EAs that exploit identical genetic operators but different population sizes to optimize the same fitness function, while the No Free Lunch Theorems compare the average performance of the two EAs over all possible fitness functions.

 \section{Analysis of Population Scalability for Elitist EAs using Global Mutation}
\label{secScalability}
\subsection{Elitist Selection and Global Mutation}
\label{sec3}
This section focuses on investigating elitist EAs that adopt global mutation and elitist selection operators. The corresponding definitions appear below.

\begin{definition}
A mutation operator is called \emph{global} if any individual can reach the optimal set via mutation after a single iteration.
\end{definition}

\begin{definition}
A selection operator is called \emph{elitist}  if the best parent individual is replaced by the best child individual only in case when the best child individual is fitter.
There is no restriction on selecting non-best individuals and any selection strategy can be applied.
\end{definition}

Global mutation guarantees that the optimal set is reachable starting from any initial state, while
elitist selection aims at maintaining the best solution found over time. An alternative elitist operator is to replace the best parent individual by a child with a better or equal fitness~\citep{he2003towards}. In the current paper we do not consider such a variant.

Now let's emphasize two virtuous properties  for mutation and elitist selection.
The first property is called the mutation property.  It compares the probability of going from a population to a higher fitness level with the probability of going from an individual to a higher fitness level.

\begin{lemma} Suppose $X=(x_1, \cdots, x_{\mu}) \in \mathcal{S}^{(\mu)}_{\non}$ is a population with $x_1=x$ being one of the best individuals. Then the following \emph{mutation property} holds for elitist EAs: for  $i=1, \cdots, \mu,$ and  $\mu \ge 2$,
\begin{align}
\label{equMutationPropertyLower}
&P_M(X, \mathcal{S}^{(\mu)}_{\high}(X)) \ge P_M(x_i, \mathcal{S}^{(1)}_{\high}(x)), \\
\label{equMutationPropertyUpper}
&P_M(X, \mathcal{S}^{(\mu)}_{\high}(X)) \le \sum^{\mu}_{i=1} P_M(x_i, \mathcal{S}^{(1)}_{\high}(x)), 
\end{align}
where   $\mathcal{S}^{(\mu)}_{\high}(x)$ denotes the set consisting of all populations whose best individual's fitness is higher than $f(x)$.

Furthermore, if a mutation operator is global, then the inequalities above are strict.
\end{lemma}
\begin{proof}
Notice that the event of going from  $X$ to a higher fitness level can be alternatively expressed as the event that at least one of individuals $x_i$ goes to a higher fitness level.

Since mutation is performed independently,
so that for $\mu \ge 2$ the probability of going from   $X$  to a higher fitness level  is
\begin{align*}
P_M( X, \mathcal{S}^{(\mu)}_{\high}  ) &= 1-\prod^{\mu}_{i=1} (1- P_M(x_i, \mathcal{S}^{(1)}_{\high}(x)),
\end{align*}
implying that for   $i=1, \cdots, \mu$,
$$P_M(X, \mathcal{S}^{(\mu)}_{\high}(X)) \ge P_M(x_i, \mathcal{S}^{(1)}_{\high}(x)).$$

The inequality for  $\mu \ge 2$
$$P_M(X, \mathcal{S}^{(\mu)}_{\high}(X)) \le \sum^{\mu}_{i=1} P_M(x_i, \mathcal{S}^{(1)}_{\high}(x))$$
follows trivially from the fact that the probability of a union of events is always bounded above by the sum of the probabilities of the constituent events.

Moreover, it is easy to see that the above inequalities are strict if the mutation operator is global.
\end{proof}

Elitist selection insures that the best individual in a population will either enter a higher fitness level or remain unchanged,  thereby never getting worse. This phenomenon is called the {elitist selection  property} and it can be reformulated as follows.
\begin{lemma} Given a population $X$ the best individual of which is $x$, the \emph{elitist selection  property} implies that
\begin{align}
\label{equElitistSelectionProperty}
P(X, \mathcal{S}^{(\mu)}_{\same}(x)) + P(X, \mathcal{S}^{(\mu)}_{\high}(x)) =1,
\end{align}
where $\mathcal{S}^{(\mu)}_{\same}(x)$ denotes the set consisting of all populations the best individual of which is $x$.
\end{lemma}

\subsection{Transition Matrices of Elitist EAs}
First consider the  $(1+1)$ elitist EA.
Arrange  all states in $\mathcal{S}^{(1)}_{\non}$ in the order of their fitness  from high to low (where the individuals at the same fitness level may be arranged  in any order), and write them in a vector form
$ ( x_1, x_2, x_3,\cdots )^T, $
where $f(x_1) \ge f(x_2) \ge f(x_3) \ge \cdots.$

Thanks to the elitist selection property discussed in the previous section, the best individual never enters a lower fitness level meaning that
for any individuals $x$ and $y$ with $f(y) \le f(x)$, the entry of the Markov transition matrix that stands for the probability of going from an individual  $x$ to an individual $y$,
$$
P(x,y) = 0.
$$

It follows then that the transition matrix $\mathbf{Q}^{(1)}$ is lower triangular and can be written in the following form:
\begin{equation}
\label{equ19}
\mathbf{Q}^{(1)} =\begin{pmatrix}
P(x_1,x_1) &  0  &  0  & \cdots &    0  \\
P(x_2, x_1)  &  P(x_2, x_2)  &0&\cdots &  0  \\
\vdots & \vdots  & \vdots&\vdots &\vdots
\end{pmatrix}.
\end{equation}

The following simple fact now follows naturally  from the definition of eigenvalues and spectral radius \citep[p.490]{meyer2000matrix}.
\begin{lemma}
\label{lem3} Given transition matrix  $\mathbf{Q}^{(1)}$   and let  $x_{\rho}$ be the state  such that
  \begin{align}
   \label{equ20}
    x_{\rho} =\arg\max\{ P(x,x); x \in \mathcal{S}^{(1)}_{\non}\}.
  \end{align}
  Then the spectral radius
  \begin{equation}
   \rho(\mathbf{Q}^{(1)})=   P(x_{\rho},x_{\rho}) .
  \end{equation}
\end{lemma}

Here $P(x,x)$ is the probability of the Markov chain remaining in state $x$.
The above lemma shows that the spectral radius of a Markov transition submatrix, that models a $(1+1)$ EA (restricted to the non-optimal states), is the maximal self-transition probability of a non-optimal state remaining unchanged in the next generation.

Next we consider a $(\mu+\mu)$ EA  (where $\mu \ge 2$).
Arrange  all populations in $\mathcal{S}^{(\mu)}_{\non}$ in the order of the fitness  of their best individual from high to low (where  populations with the same best individual are arranged  together in an arbitrary order), and  write them in a vector form:
 $( X_1, X_2, \cdots  )^T.$
If $x_1, x_2, \cdots$ denote the corresponding best individuals, then their fitness decreases: $f(x_1) \ge f(x_2) \ge \cdots .$

Once again, thanks to the elitist selection property, the best individual in a population  never revisits any state at a lower fitness level.
Thus   the   probability  of going from population $X$ (with the best individual $x$) to population $Y$ that is not in the set  $\mathcal{S}^{(\mu)}_{\same}(x)$ or $\mathcal{S}^{(\mu)}_{\high}(x)$ is
$0$. It follows then that the matrix $ \mathbf{Q}^{(\mu)}$ is a block lower triangular matrix, that can be written in the form:
\begin{equation}\label{equ22}
\mathbf{Q}^{(\mu)}=
\begin{pmatrix}
 \mathbf{Q}^{(\mu)}_{x_1,x_1} & \mathbf{O} & \mathbf{O} & \cdots   \\
 \mathbf{Q}^{(\mu)}_{x_1,x_2}  &  \mathbf{Q}^{(\mu)}_{x_2, x_2} & \mathbf{O} &\cdots   \\
 \vdots & \vdots  & \vdots &\vdots   \\
 \end{pmatrix},
\end{equation}
where $\mathbf{O}$ denotes a zero  matrix and $\mathbf{Q}^{(\mu)}_{x,y}$ is the  submatrix  consisting of transition probabilities from the states in $\mathcal{S}^{(\mu)}_{\same} (x)$ to the states in $\mathcal{S}^{(\mu)}_{\same} (y)$. In particular, $\mathbf{Q}^{(\mu)}_{x,x}$ is the submatrix consisting of transition probabilities within the set $\mathcal{S}^{(\mu)}_{\same} (x)$.

The following lemma is an extension of Lemma \ref{lem3} for the case when $\mu\ge 2$.
\begin{lemma}
\label{lem4}
Let $\mathbf{Q}^{(\mu)}_{x,x}$ denote the diagonal block submatrix as in (\ref{equ22}), that represents transitions within the set $\mathcal{S}^{(\mu)}_{\same} (x) $,
then
  \begin{align}
        &\rho(\mathbf{Q}^{(\mu)}) =\max_{x \in \mathcal{S}^{(1)}_{\non}}  \rho (\mathbf{Q}^{(\mu)}_{x,x}) .
      \end{align}
\end{lemma}

\begin{proof}
The proof  is based on a simple fact  \citep[Exercise 7.1.4]{meyer2000matrix}: if $\mathbf{A}$ is a block lower triangular matrix such that
\begin{equation}
\mathbf{A} =
\begin{pmatrix}
\mathbf{A}_{1,1}  & \mathbf{O}    \\
\mathbf{A}_{2,1} &  \mathbf{A}_{2,2}
\end{pmatrix},
\end{equation}
then $\lambda$ is an eigenvalue of $\mathbf{A}$ if and only if $\lambda$ is an eigenvalue of $\mathbf{A}_{1,1}$ or $\mathbf{A}_{2,2}$.  Thus
$\rho(\mathbf{A}) =\max\{ \rho(\mathbf{A}_{1,1}), \rho(\mathbf{A}_{2,2})\}.$

Let  $ \mathbf{A}  =\mathbf{Q}^{(\mu)}$ (see the matrix (\ref{equ22})),
then
\begin{align*}
        \rho(\mathbf{Q}^{(\mu)}) = \max_{x \in \mathcal{S}^{(1)}_{\non}}  \rho (\mathbf{Q}^{(\mu)}_{x,x})
\end{align*}
as claimed.
\end{proof}

The following lemma provides lower and upper bounds on the spectral radius of the  transition probability submatrix  $\mathbf{Q}^{(\mu)}_{x,x}$.

\begin{lemma} \label{lemSpectralBoundsTran} The spectral radius of the transition submatrix $\mathbf{Q}^{(\mu)}_{x,x}$ satisfies
\begin{align}
\min_{ X \in  \mathcal{S}^{(\mu)}_{\same}(x)   }     {P(X, \mathcal{S}^{(\mu)}_{\same}(x))}  \le \rho(\mathbf{Q}^{(\mu)}_{x,x}) \le \max_{ X \in  \mathcal{S}^{(\mu)}_{\same}(x)  }     { P(X, \mathcal{S}^{(\mu)}_{\same}(x))}.
\end{align}
\end{lemma}

\begin{proof}
The proof is the same as that of Lemma~\ref{lemSpectralBoundsFund}. Indeed, substituting $\mathbf{Q}^{(\mu)}_{x,x}$ in place of $\mathbf{A} $ yields the desired conclusion.
\end{proof}


\subsection{$\rho$-Scalability Always Happens for Elitist EAs Exploiting Global Mutation}
An intuitive reason behind the use of population-based EAs is that a larger population size is likely to increase the convergence rate. The following proposition proves that this is, indeed, the case.

\begin{theorem}
\label{theScalability} Suppose a $(1+1)$ elitist EA and a $(\mu+\mu)$ elitist EA  (where $\mu \ge 2$) exploit identical global mutation operator to maximize the same fitness function. Then
$$
\mbox{$\rho$-scalability($\mu$)} >1.
$$
\end{theorem}

\begin{proof}
For the $(1+1)$ elitist EA, let $x_{\rho} \in \mathcal{S}^{(1)}_{\non}$ be an individual such that
$
\rho(\mathbf{Q}^{(1)})=P(x_{\rho},x_{\rho}).
$

Likewise, for a $(\mu+\mu)$ elitist EA, Lemma~\ref{lem4} says that
$$
\rho(\mathbf{Q}^{(\mu)})=\max_{x \in \mathcal{S}^{(1)}_{\non}}\rho(\mathbf{Q}^{(\mu)}_{x,x}).
$$

Since the set $\mathcal{S}^{(1)}_{\non}$ is finite, there exists an $ x \in \mathcal{S}^{(1)}_{\non}$ such that
\begin{equation}\label{xspectralradius}
 \rho(\mathbf{Q}^{(\mu)})=\rho(\mathbf{Q}^{(\mu)}_{x,x}).
\end{equation}

Consider the transition probability matrix $\mathbf{Q}^{(\mu)}_{x,x}$ for such an $x$.

According to Lemma~\ref{lemSpectralBoundsTran}, the spectral radius  $\rho(\mathbf{Q}^{(\mu)}_{x,x})$  is bounded above as
$$
 \rho(\mathbf{Q}^{(\mu)}_{x,x})  \le  \max_{X \in \mathcal{S}^{(\mu)}_{\same}(x)}  P(X,\mathcal{S}^{(\mu)}_{\same}(x) ).
$$

Thus there exists an  $X=(x_1, x_2, \cdots, x_{\mu})$ in  the set $\mathcal{S}^{(\mu)}_{\same}(x)$ where $x_1=x$, and the above inequality holds for this specific $X$,
$$
 \rho(\mathbf{Q}^{(\mu)}_{x,x})  \le     P(X,\mathcal{S}^{(\mu)}_{\same}(x) ).
$$

Combining the inequality above with the  elitist selection  property~ (\ref{equElitistSelectionProperty}),
\begin{align*}
   P(X,\mathcal{S}^{(\mu)}_{\high}(x) ) + P(X,\mathcal{S}^{(\mu)}_{\same}(x) )=1,
\end{align*}
yields
\begin{align}
\label{rho-upper-bound-a}
 \rho(\mathbf{Q}^{(\mu)}_{x,x})
\le 1-P(X,\mathcal{S}^{(\mu)}_{\high}(x) ) .
\end{align}

Now the global  mutation property   tells us that, for $ \mu\ge 2$,
\begin{align*}
 P(X,\mathcal{S}^{(\mu)}_{\high}(x) )  >     P(x,\mathcal{S}^{(1)}_{\high}(x) )
  = 1- P(x,x).
  \end{align*}

Recall that $P(x_{\rho},x_{\rho})$ is the maximal self-transition probability so that
\begin{align*}
  P(X,\mathcal{S}^{(\mu)}_{\high}(x) ) \ge  1-P(x_{\rho},x_{\rho})
 =   1- \rho(\mathbf{Q}^{(1)} ).
\end{align*}
Substituting this bound into Inequality (\ref{rho-upper-bound-a}) yields
\begin{align*}
       \rho(\mathbf{Q}^{(\mu)}_{x,x}) <   \rho(\mathbf{Q}^{(1)} ).
\end{align*}

Recalling (\ref{xspectralradius}):
 $\rho(\mathbf{Q}^{(\mu)})=\rho(\mathbf{Q}^{(\mu)}_{x,x}),$
we deduce that
\begin{align*}
       &\rho(\mathbf{Q}^{(\mu)}) <   \rho(\mathbf{Q}^{(1)} ),\end{align*}
       so that $\mbox{$\rho$-scalability($\mu$)}    > 1.$ Thereby establishing the desired conclusion.
\end{proof}


\subsection{$\infty$-Scalability May Not Happen  for Elitist EAs using Global Mutation}
Another intuitive reason behind the use of population-based EAs is that a larger population size is likely to shorten the expected number of generations. Unfortunately sometimes this is wrong. The following example  shows  that increasing the population size  increases, rather than reduces, the maximum of the expected number of generations  for an EA to find an optimal solution. Equivalently, $\infty$-scalability$(\mu)<1$  for any $\mu \ge 2$.

The fitness function is given in Table~\ref{tab1}.
\begin{table}[ht]
\begin{center}
\begin{tabular}{c|c|c|c|c|c}
\hline
state & $x_0$ & $x_1$ & $x_2$ & $x_3$ & $x_4$\\
\hline
fitness & 5   & 4 & 3 & 2 & 1\\
\hline
\end{tabular}
\end{center}
\caption{Fitness function.}
\label{tab1}
\end{table}

Consider the following $(\mu+\mu)$ EA (=$(1+\mu)$ EA): the best individual is replicated by $\mu$ copies and each copy generates a child via mutation; the best individual is replaced only when a child is better than it. The mutation transition probabilities  $P(x,y), x, y =x_0, \cdots, x_4$ are given in Table~\ref{tab2}, where $\epsilon \geq 0$ is a sufficiently small constant (the size of $\epsilon$ will be discussed later). When $\epsilon >0$, the mutation operator is global.
\begin{table}[ht]
\begin{center}
\begin{tabular}{c|c|c|c|c|c}
\hline
state & $x_0$ & $x_1$ & $x_2$ & $x_3$ & $x_4$\\
\hline
$x_0$ & $1-4\epsilon$ & $\epsilon$  & $\epsilon$ & $\epsilon$ & $\epsilon$ \\
$x_1$ & $1-4\epsilon$ & $\epsilon$  & $\epsilon$ & $\epsilon$ & $\epsilon$ \\
$x_2$ & $\epsilon$ & $1-4\epsilon$ & $\epsilon$ & $\epsilon$ & $\epsilon$ \\
$x_3$ & $1-4 \epsilon$ & $\epsilon$ & $\epsilon$ & $\epsilon$ & $\epsilon$\\
$x_4$ & $\epsilon$ & $\epsilon$ & $0.5$ & $0.5-3 \epsilon$ & $\epsilon$\\
\hline
\end{tabular}
\end{center}
\caption{Mutation transition probability matrix.}
\label{tab2}
\end{table}

First, set $\epsilon=0$. The maximal value of the expected number of generations for the $(1+1)$ EA to encounter the optimal individual $x_0$ is the time when the EA starts from $x_4$. According to the probability transition matrix above, it takes $2$ time steps (with probability $1$) to reach the optimal state along the road $x_4 \to x_3 \to x_0$ (with probability $0.5$), while it takes $3$ time steps along the road $x_4 \to x_2 \to x_1 \to x_0$ (also with probability $0.5$). Thus, by definition of the expectation,
$$
m^{(1)}(x_4) =2 \cdot 0.5 +3 \cdot 0.5 = 2.5.
$$
The maximal value of the expected number of generations for the $(2+2)$ to reach a population containing the optimal individual $x_0$ is the time when the EA starts from $ (x_4, x_4)$. Because of the elitist selection, the only possible offspring population is $ (x_3, \, x_3)$ or $(x_2, \, x_2)$. Since the event of going from $(x_4,x_4)$  to $(x_3, \, x_3)$ happens only if both $x_4$s mutate into $x_3$, the probability of this event is $0.5 \cdot 0.5 = 0.25$. Consequently, the probability that the event of going from $(x_4,x_4)$  to $(x_2, \, x_2)$ happens is $1-0.25 =0.75$. Thus, according to the mutation probability transition matrix, the EA reaches the optimum along the road $(x_4,x_4) \to (x_3,x_3) \to (x_0, x_0)$  with probability $0.25$, while it does so along the road $(x_4,x_4) \to (x_2,x_2) \to (x_1,x_1) \to (x_0, x_0)$ with probability $0.75$, so that
$$
m^{(2)}(x_4, x_4) =2 \cdot 0.25 +3 \cdot 0.75 = 2.75.
$$

This demonstrates explicitly that the maximal value of the expected number of generations that the $(1+1)$ EA needs to reach an optimum is shorter than that the $(2+2)$ EA needs, since
\begin{equation}\label{counterExIneq}
m^{(1)}(x_4) =2.5< m^{(2)}(x_4, x_4)=2.75.
\end{equation}

Furthermore, the reasoning above generalizes to the case when $\mu > 2$ and shows that $$m^{(\mu)}( x_4,   \ldots, x_4 ) = 2 \cdot 0.5^{\mu} + 3 \cdot (1 - 0.5^{\mu}),$$ thereby demonstrating that $m^{(\mu)}({x_4,   \ldots, x_4})$ is a strictly increasing function of the population size $\mu$ (increasing the population size also increases the maximal value of the expected number of generations).

Now observe that $$m^{(1)}(x_4) - m^{(\mu)}(x_4, x_4)$$ is a continuous function of $\epsilon$ so that for small enough $\epsilon$ Inequality~(\ref{counterExIneq}) as well as the conclusion in the paragraph above still hold. Moreover, notice that the continuity argument implies that the elitist selection can also be alleviated to certain non-elitist selection  (non-best
individuals may replace the parent individual but with tiny probabilities) so that all the same conclusions remain valid.

\subsection{$a$-Scalability and $\hat{a}$-Scalability May Not Happen  for Elitist EAs using Global Mutation}
The following modification of the example in the previous subsection shows that increasing the population size may increase, rather than reduce, the average value of the expected number of generations, regardless of whether the population is chosen uniformly at random from the set of all transient states or from the set of all possible states. Equivalently,
$a$-scalability$(\mu)<1$ and $\hat{a}$-scalability$(\mu)<1$ for some  $\mu\ge 2$.

The fitness function is given in Table~\ref{tab3}.
\begin{table}[ht]
\begin{center}
\begin{tabular}{c|c|c|c|c|c}
\hline
state & $x_0$ & $x_1$ & $x_2$ & $x_3$ & $x_{i} (i =4, \cdots, 103)$ \\
\hline
fitness & 5   & 4 & 3 & 2 & 1\\
\hline
\end{tabular}
\end{center}
\caption{Fitness function.}
\label{tab3}
\end{table}

Consider the following $(\mu+\mu)$ EA (=$(1+\mu)$ EA): the best individual is replicated by $\mu$ copies and each copy generates a child via mutation; the best individual is replaced only when a child is better than it. The mutation transition probabilities appear in Table~\ref{tab4}, where $\epsilon \geq 0$ is a sufficiently small constant just as in the previous example.
\begin{table}[ht]
\begin{center}
\begin{tabular}{c|c|c|c|c|c}
\hline
state & $x_0$ & $x_1$ & $x_2$ & $x_3$ & $x_{i} (i \ge 4)$\\
\hline
$x_0$ & $1-4\epsilon$ & $\epsilon$  & $\epsilon$ & $\epsilon$ & $0.01\epsilon$ \\
$x_1$ & $1-4\epsilon$ & $\epsilon$  & $\epsilon$ & $\epsilon$ & $0.01\epsilon$ \\
$x_2$ & $\epsilon$ & $1-4\epsilon$ & $\epsilon$ & $\epsilon$ & $0.01\epsilon$ \\
$x_3$ & $1-4 \epsilon$ & $\epsilon$ & $\epsilon$ & $\epsilon$ & $0.01\epsilon$\\
$x_{i} (i \ge 4)$ & $\epsilon$ & $\epsilon$ & $0.5$ & $0.5-3 \epsilon$ & $0.01\epsilon$\\
\hline
\end{tabular}
\end{center}
\caption{Mutation transition probability matrix.}
\label{tab4}
\end{table}

The only difference from the previous example is that now there are $100$ ``bad'' states with the same largest expected number of generations (rather than only a single state, as in the previous example), compared with only $4$ ``good'' states with a smaller   expected number of generations. Thus the average value of the expected number of generations is almost the same as the maximal value of the expected number of generations. The parameter $\epsilon$ can be chosen sufficiently small and also elitist selection can be alleviated according to the same type of continuity argument as in the previous example, of course, to show that the $(1+1)$ EA outperforms the $(2+2)$, $(3+3)$ or $(4+4)$ EA.

\section{General Studies: Conditions for Superlinear $\rho$-Scalability to Take Place}
\label{secGeneralStudy}
\subsection{General Condition for Superlinear $\rho$-Scalability to Take Place}
In the current subsection we present a rather general sufficient and necessary condition for  superlinear scalability to take place that applies to both elitist and non-elitist EAs.   The condition is based on the concept of a ``road''. Intuitively, a road is a transition path between two states $X$ and $Y$. A rigorous definition appears below \citep{he1995convergence}.
\begin{definition}
Given two states $X, Y \in \mathcal{S}^{(\mu)}$, if there exists $k$ states $ X_0=X \to X_1 \to  \cdots \to  X_k=Y $ such that
$$
P(X_0, X_1) \cdots P(X_{k-1}, X_{k+1})>0,
$$  then $\{X_0, \cdots, X_k \}$ is called a \emph{road   from $X$ to $Y$}. We also say that $k$ is the length of the road. We write $road(X,  Y, k)$ to denote the set of all roads from $X$ to $Y$ having length $k$.
\end{definition}

Let $P(road(X, \mathcal{S}^{(\mu)}_{\opt}(x), k) )$ denote the probability of going from $X$ to the set $\mathcal{S}^{(\mu)}_{\opt}(x)$ via ``roads'' of length $k$.

A general sufficient and necessary condition for  superlinear scalability to take place appears in the following theorem. The theorem is largely based on the classical Gelfand's spectral radius formula.
\begin{theorem}
\label{theSuffNece}
Suppose we are given a $(1+1)$  EA and a $(\mu+\mu)$ EA  (where $\mu \ge 2$) that exploit identical mutation operator to maximize the same fitness function. For the $(\mu+\mu)$ EA,  supperlinear scalability  happens if and only if
there exists some
$k>0$ and for $X \in \mathcal{S}^{(\mu)}_{\non} $,
\begin{align}
   P (  \mbox{$road(X, \mathcal{S}^{(\mu)}_{\opt}, k)$  })
>   1- \left( 1- \mu (1-\rho(\mathbf{Q}^{(1)})) \right)^k.\label{equGneneralCondition}
\end{align}
\end{theorem}

\begin{proof}
(1) The proof that the condition is sufficient.

Suppose Inequality (\ref{equGneneralCondition}) holds for all populations $X$   in the set $\mathcal{S}^{(\mu)}_{\non}$. It follows from the assumption that
\begin{align*}
 \max_{X \in \mathcal{S}^{(\mu)}_{\non}}  P(\Phi_k \in \mathcal{S}^{(\mu)}_{\non} \mid \Phi_0= X )
 <    \left( 1- \mu (1-\rho(\mathbf{Q}^{(1)}) \right)^k.
\end{align*}

Rewriting the inequality above in terms of the $\infty$-norm, we obtain,
\begin{align}
\label{equNormBound}
 \parallel (\mathbf{Q}^{(\mu)})^k \parallel_{\infty}  < \left( 1- \mu (1-\rho(\mathbf{Q}^{(1)}) \right)^k.
\end{align}

Since the spectral radius of a matrix is not bigger than its maximum norm  \citep[p.619]{meyer2000matrix}, 
\begin{align*}
 \rho((\mathbf{Q}^{(\mu)})^k) \le  \parallel (\mathbf{Q}^{(\mu)})^k \parallel_{\infty}  ,
\end{align*}
so that
\[
 \rho(\mathbf{Q}^{(\mu)}) =\left(  \rho((\mathbf{Q}^{(\mu)})^k) \right)^{1/k} \le \left(\parallel (\mathbf{Q}^{(\mu)})^k \parallel_{\infty} \right)^{1/k}.
\]
Combining Inequality (\ref{equNormBound}) with the above inequality yields
\begin{align*}
 &\rho(\mathbf{Q}^{(\mu)}) <  1- \mu \left(1-\rho(\mathbf{Q}^{(1)}) \right),\\
 & \frac{1-\rho(\mathbf{Q}^{(\mu)})}{1-\rho(\mathbf{Q}^{(1)})} > \mu .
\end{align*}
This means that superlinear scalability takes place.

(2) The proof that the condition is necessary.

Suppose Inequality (\ref{equGneneralCondition}) does not hold. This means that
for any $  k>0$, there exists some  $X  \in \mathcal{S}^{(\mu)}_{\non}$ such that
\begin{align}
  P(  \mbox{$road(X , \mathcal{S}^{(\mu)}_{\opt}, k)$  })
\le   1- \left( 1- \mu (1-\rho(\mathbf{Q}^{(1)}) \right)^k.\label{equOneMax7}
\end{align}

It follows then that for any $k>0$  
\begin{align*}
&P(\Phi_k \in \mathcal{S}^{(\mu)}_{\opt}  \mid \Phi_0= X )
  \le 1-  \left( 1- \mu (1-\rho(\mathbf{Q}^{(1)}))\right)^k,\\
 &P(\Phi_k \in \mathcal{S}^{(\mu)}_{\non}  \mid \Phi_0= X )
  \ge   \left( 1- \mu (1-\rho(\mathbf{Q}^{(1)}))\right)^k,
\end{align*}
and this, in turn, implies that
\begin{align*}
 \max_{Y \in \mathcal{S}^{(\mu)}_{\non}  } P(\Phi_k \in \mathcal{S}^{(\mu)}_{\non}   \mid \Phi_0= Y )
\ge     \left( 1- \mu (1-\rho(\mathbf{Q}^{(1)})) \right)^k.
\end{align*}

Rewriting the inequality above in terms of the $\infty$-norm we deduce that,
\[
 \parallel (\mathbf{Q}^{(\mu)})^k \parallel_{\infty} \ge  \left( 1- \mu (1-\rho(\mathbf{Q}^{(1)})) \right)^k.
\]

Taking the limit as $k \to +\infty$, and applying Gelfand's spectral radius formula, we obtain:
\begin{align*}
 &\rho(\mathbf{Q}^{(\mu)}) = \lim_{k\to \infty} \left(\parallel (\mathbf{Q}^{(\mu)})^k \parallel_{\infty} \right)^{1/k},
\end{align*}
so that
\begin{align*}
&\rho(\mathbf{Q}^{(\mu)})  \ge  1- \mu (1-\rho(\mathbf{Q}^{(1)})),\\
 &\frac{1-\rho(\mathbf{Q}^{(\mu)}) }{1-\rho(\mathbf{Q}^{(1)}) }  \le \mu.
\end{align*}

This means that no superlinear scalability takes place.
\end{proof}

\subsection{Sufficient  and Necessary Condition for Superlinear $\rho$-Scalability to Happen for Elitist  EAs}
The sufficient and necessary condition for the superlinear scalability to occur that has been established in Theorem~\ref{theSuffNece} can be reformulated in a more explicit fashion when dealing with elitist EAs. We call this reformulation ``road through bridge''. A detailed analysis is provided in the current subsection.

\begin{definition}
An individual $y$ is called a \emph{bridgeable point} of an individual $x$ if  $y$ satisfies the following conditions:
\begin{enumerate}
  \item the fitness of $x$ is larger than that of $y$: $f(x) \ge f(y)$;
  \item the probability of going from   $x$ to the set $\mathcal{S}^{(1)}_{\high}(x)$ via mutation is smaller than that from   $y$ to the same set $\mathcal{S}^{(1)}_{\high}(x)$.
  \begin{equation}
  \begin{array}{lll}
   P_M(x,  \mathcal{S}^{(1)}_{\high}(x)) \le P_M(y,  \mathcal{S}^{(1)}_{\high}(x)).
  \end{array}
  \end{equation}
\end{enumerate}
\end{definition}

The term ``bridgeable point'' is motivated by the following intuitive notion:  $y$ may serve  as a ``bridge" for $x$ to step towards a higher fitness level.

To achieve superlinear scalability, it is important for elitist EAs to go through  some ``bridgeable  point''.  Intuitively   there are two types of roads going from a state towards a higher fitness level. One is the road  going from its current  fitness level directly towards the higher fitness level; another type is the road  through some bridgeable point  before reaching a higher fitness level.

Given a  population $X$ the best individual of which is $x$ and a population $Y$ in the set $\mathcal{S}^{(\mu)}_{\high}(x)$, the roads from $X$ to $Y$  can be classified into two categories:
 \begin{itemize}
 \item  \emph{Road through bridge $\{X_0=X, X_1, \cdots, X_{k-1}, X_k=Y\}$}:   at least one of the intermediate populations $X_1, \cdots, X_{k-1}$  contains a bridgeable point of $x$.

\item  \emph{Road  over gap $\{X_0=X, X_1, \cdots, X_{k-1}, X_k=Y\}$}:    none of the intermediate  populations $X_1, \cdots, X_{k-1}$  contains  a bridgeable point of $x$.
\end{itemize}

Let $P (road(X, \mathcal{S}^{(\mu)}_{\high}(x), k) \mbox{ through bridge})$ denote the probability of going from $X$ to the set $\mathcal{S}^{(\mu)}_{\high}(x)$ via ``roads through bridge'' of length $k$.

Likewise, let $P (road(X, \mathcal{S}^{(\mu)}_{\high}(x), k) \mbox{ over gap})$ denote the probability of going from $X$ to the set $\mathcal{S}^{(\mu)}_{\high}(x)$ via ``roads over gap'' of length $k$.

The following theorem provides a sufficient and necessary condition for  superlinear scalability to occur in case of elitist EAs in terms of the "road through bridge''.
\begin{theorem}
\label{theSuffNeceElitist}
Suppose we are given a $(1+1)$ elitist EA and a $(\mu+\mu)$ elitist EA  (where $\mu \ge 2$) that exploit identical mutation operator to maximize the same fitness function. Consider an individual $x_{\rho}$ such that
\[
x_{\rho}=\arg\max\{ P(x,x); x \in \mathcal{S}^{(1)}_{\non}\}.
\]

For the $(\mu+\mu)$ EA, supperlinear $\rho$-scalability  happens if and only if the following \emph{road through bridge}  condition holds:
there exists some $k>0$, such that for any  $x \in \mathcal{S}^{(1)}_{\non}$ and any $ X \in \mathcal{S}^{(\mu)}_{\same}(x)$,
\begin{align}
  &P(  \mbox{$road(X, \mathcal{S}^{(\mu)}_{\high}(x), k)$  over gap}) \nonumber
 +  P( \mbox{$road(X, \mathcal{S}^{(\mu)}_{\high}(x), k)$  through bridge}) \nonumber \\
>  &1- \left( 1- \mu (1-P(x_{\rho},x_{\rho}) \right)^k.\label{equOneMax6}
\end{align}
\end{theorem}

\begin{proof}
For the $(1+1)$ elitist EA, Lemma~\ref{lem3} tells us that
\[
    \rho(\mathbf{Q}^{(1)}) = P(x_{\rho},x_{\rho}).
\]

For the $(\mu+\mu)$ EA, from Lemma~\ref{lem4}, it follows that
\begin{align*}
    \rho(\mathbf{Q}^{(\mu)}) =\max_{x \in \mathcal{S}^{(1)}_{\non}}\rho( \mathbf{Q}^{(\mu)}_{x,x}).
\end{align*}

Since the set $\mathcal{S}^{(1)}_{\non}$ is finite, there is some $x \in \mathcal{S}^{(1)}_{\non}$ such that
\begin{align}
\label{max-rho-sufficient}
    \rho(\mathbf{Q}^{(\mu)}) = \rho( \mathbf{Q}^{(\mu)}_{x,x}).
\end{align}

Thanks to elitist selection, Inequality~(\ref{equOneMax6}) is equivalent to saying that for any $X \in \mathcal{S}^{(\mu)}_{\same}(x)$
\begin{align*}
 P ( road(X, \mathcal{S}^{(\mu)}_{\high}(x), k)
 >  1- \left( 1- \mu (1-\rho(\mathbf{Q}^{(1)}) \right)^k.
\end{align*}

The desired conclusion now follows directly from Theorem~\ref{theSuffNece} applied to the Markov transition submatrix $\mathbf{Q}^{(\mu)}_{x,x}$ and replacing $\mathcal{S}^{(\mu)}_{\non}$ and $\mathcal{S}^{(\mu)}_{\opt}$ in Theorem~\ref{theSuffNece} by $\mathcal{S}^{(\mu)}_{\same}(x)$ and $\mathcal{S}^{(\mu)}_{\high}(x)$ respectively.
\end{proof}

\subsection{Necessary Condition for   Superlinear $\rho$-Scalability to Happen for Elitist  EAs}

The following theorem informs us further that the existence of a ``road through bridge'' is a necessary condition for  superlinear scalability to take place. Let $\mathcal{S}^{(\mu)}_{\bridge}(x)$ denote the set of all populations  which contain  a bridgeable point of $x$.

\begin{theorem}
\label{theBridgeNecessary}
Suppose we are given a $(1+1)$ elitist EA and a $(\mu+\mu)$ elitist EA  (where $\mu \ge 2$) that exploit identical mutation operator to maximize the same fitness function. Consider an individual $x_{\rho}$ such that
\[
x_{\rho}=\arg\max\{ P(x,x); x \in \mathcal{S}^{(1)}_{\non}\}.
\]

Let $
X_{\rho} =(x_{\rho}, \cdots, x_{\rho})$. If for some  population size $\mu\ge2$,   and for any $k>0$
\begin{equation}\label{equBridgeNecessary}
 P( \mbox{   $Road(X_{\rho}, \mathcal{S}^{(\mu)}_{\high}(x_{\rho}), k)$  through bridge}) =0,
\end{equation}
then no supperlinear $\rho$-scalability ever takes place for such a $\mu$.
\end{theorem}

\begin{proof}
For the $(1+1)$ elitist EA, from Lemma~\ref{lem3},
\[
    \rho(\mathbf{Q}^{(1)}) = P(x_{\rho},x_{\rho}).
\]

For a $(\mu+\mu)$ elitist EA (where $\mu\ge 2$),  consider the Markov transition probability submatrix $\mathbf{Q}^{(\mu)}_{x_{\rho},x_{\rho}}$.

Split the set $\mathcal{S}^{(\mu)}_{\same}(x_{\rho})$  into two subsets:
 $\mathcal{S}^{(\mu)}_{\bridge}(x_{\rho})$  and  $\mathcal{S}^{(\mu)}_{\same}(x_{\rho}) \setminus \mathcal{S}^{(\mu)}_{\bridge}(x_{\rho})$.

From the condition of the theorem, for any $k >0$,
\begin{equation*}
 P( \mbox{$Road(X_{\rho}, \mathcal{S}^{(\mu)}_{\high}(x_{\rho}), k)$  through bridge}) =0,
\end{equation*}
so that the probability of going from any state in the set $\mathcal{S}^{(\mu)}_{\same}(x_{\rho}) \setminus \mathcal{S}^{(\mu)}_{\bridge}(x_{\rho})$ to the set $\mathcal{S}^{(\mu)}_{\bridge}(x_{\rho})$ is $0$.

Hence the matrix $\mathbf{Q}^{(\mu)}_{x_{\rho},x_{\rho}}$ is reducible. We write it in the following form

\[
\begin{pmatrix}
\mathbf{\hat{Q}}^{(\mu)}_{x_{\rho},x_{\rho}} & \mathbf{{O}} \\
  *  & **
\end{pmatrix},
\]
where
$\mathbf{\hat{Q}}^{(\mu)}_{x_{\rho},x_{\rho}}$ represents the transition probability submatrix within the set $\mathcal{S}^{(\mu)}_{\same}(x_{\rho}) \setminus  \mathcal{S}^{(\mu)}_{\bridge}(x_{\rho})$. The $*$ part represents transition probabilities from the set $\mathcal{S}^{(\mu)}_{\bridge}(x_{\rho})$ to the set $\mathcal{S}^{(\mu)}_{\same}(x_{\rho}) $, and the part labelled with the $**$ symbol stands for the transition probabilities within the set $\mathcal{S}^{(\mu)}_{\bridge}(x_{\rho})$. $\mathbf{O}$ denotes a zero matrix.

For any population $X=(x_1, \cdots, x_{\mu})$ in the set $\mathcal{S}^{(\mu)}_{\same}(x_{\rho}) \setminus \mathcal{S}^{(\mu)}_{\bridge}(x_{\rho})$, since none of its individuals is a  ``bridgeable point'' of $x_{\rho}$, we have
\begin{align*}
P_M(x, \mathcal{S}^{(1)}_{\high}(x_{\rho})) & \le P_M(x_{\rho}, \mathcal{S}^{(1)}_{\high}(x_{\rho}))=1-\rho(\mathbf{Q}^{(1)}).
\end{align*}
According to the mutation property (\ref{equMutationPropertyUpper}), for any $\mu >2$,
\[
P(X, \mathcal{S}^{(\mu)}_{\high}(x_{\rho})) < \mu (1-\rho(\mathbf{Q}^{(1)})) .
\]

Since the inequality holds for any population in the set $\mathcal{S}^{(\mu)}_{\same}(x_{\rho}) \setminus \mathcal{S}^{(\mu)}_{\bridge}(x_{\rho})$, we have
\[
\min_{X \in \mathcal{S}^{(\mu)}_{\same}(x_{\rho}) \setminus \mathcal{S}^{(\mu)}_{\bridge} (x_{\rho}))} P (X, \mathcal{S}^{(\mu)}_{\same}(x_{\rho}) \setminus \mathcal{S}^{(\mu)}_{\bridge} (x_{\rho})) \ge 1-\mu (1-\rho(\mathbf{Q}^{(1)})).
\]

According to Lemma~\ref{lemSpectralBoundsTran},
\begin{align*}
\rho(\mathbf{\hat{Q}}^{(\mu)}_{x_{\rho},x_{\rho}})\ge 1- \mu (1-\rho(\mathbf{Q}^{(1)})).
\end{align*}
Since
$\rho(\mathbf{Q}^{(\mu)}) \ge \rho(\mathbf{\hat{Q}}^{(\mu)}_{x_{\rho},x_{\rho}})$, we deduce that \begin{align*}
1-\rho(\mathbf{Q}^{(\mu)}) \le \mu (1-\rho(\mathbf{Q}^{(1)})),
\end{align*}
which means that no superlinear $\rho$-scalability takes place.
\end{proof}

From the theorem, we deduce two necessary conditions for  superlinear $\rho$-scalability to happen as follows. First of all,
bridgeable point(s) must exist. Furthermore, bridgeable point(s) must be preserved during selection. Consequently, if there is no population diversity, then there is no superlinear $\rho$-scalability.

\section{Case Study 1: No Superlinear $\rho$-Scalability on   Non-bridgeable Fitness Landscapes}
\label{secCaseStudy1}

\subsection{Superlinear $\rho$-Scalability Never Happens to  Non-bridgeable Fitness Landscapes}

\begin{definition}
Given a fitness function $f(x)$, we say that the  fitness landscape assocated with a $(1+1)$ EA is \emph{non-bridgeable} meaning that for any two non-optimal states $x$ and $y$, if $x$ has a better fitness than  $y $ (that is $f(x) \ge f(y)$), then starting from $x$, the EA has a larger probability to enter a higher fitness level, $\mathcal{S}^{(1)}_{\high}(x)$ than starting from $y$ (and arriving at the same subset $\mathcal{S}^{(1)}_{\high}(x)$), that is 
\begin{align*}
 P_M(x, \mathcal{S}^{(1)}_{\high}(x)  )
>  P_M(y, \mathcal{S}^{(1)}_{\high}(x)).
\end{align*}
\end{definition}

In terms of the average convergence rate, using a population delivers no superlinear scalability   if the fitness landscape  associated with the $(1+1)$ EA is  non-bridgeable. The following theorem demonstrates this.
\begin{proposition}
  \label{proNon-bridgeable}
Suppose we are given a $(1+1)$ elitist EA and a $(\mu+\mu)$ elitist EAs (where $\mu \ge 2$) that exploit identical mutation operator for maximizing the same fitness function. If the fitness landscape associated with the $(1+1)$ EA is  non-bridgeable, then no superlinear $\rho$-scalability happens.
\end{proposition}

\begin{proof}
Let $x_{\rho}$ be an individual such that
\[
x_{\rho}=\arg\max\{ P(x,x); x \in \mathcal{S}^{(1)}_{\non}\},
\]
and $
X_{\rho} =(x_{\rho}, \cdots, x_{\rho})$.
It is easy to see there is no bridgeable point for $X_{\rho}$, so that for $k >0$
\begin{align*}
 P( \mbox{$Road(X_{\rho}, \mathcal{S}^{(\mu)}_{\high}(x_{\rho}), k)$  through bridge}) =0.
\end{align*}
The desired conclusion now follows directly from Theorem~\ref{theBridgeNecessary}.
\end{proof}

Analogous results have been established under $\infty$-scalability and
$a$-scalability for non-bridgeable fitness landscapes which are called monotonic fitness landscapes in \citet{he2012general}.

 \subsection{An Example of Non-bridgeable Fitness Landscapes}\label{nonBridgeableSubsectExample}
\label{exaOneMax}
Consider  the average capacity 0-1 knapsack problem~\citep{martello1990knapsack}, described as follows: let $x$ be a binary string $(s_1, \cdots, s_n) \in \{0,1\}^n$,
\begin{align*}
 \max_{x}  \sum^n_{i=1} v_i s_i  \quad
 \mbox{subject to } \sum^n_{i=1} w_i s_i \le C,
\end{align*}
where $v_i$ is the value of the $i$-th item, $w_i$ its weight and $C=0.5 \sum^n_{i=1} w_i$ is the capacity.

Consider the instance where the values and the weights of the items  $v_i = w_i =1$  for $i=1, \cdots, n$, and $C =0.5  n$.
The fitness function for this instance is similar to the One-Max function.
\begin{equation}
\label{equOneMax} f_1(x)=
\left\{
\begin{array}{ll}
\sum^n_{i=1} s_i, & \mbox{if } \sum^n_{i=1}  s_i \le 0.5 n,\\
\mbox{infeasible}, &\mbox{otherwise}.
\end{array}
\right.
\end{equation}

An individual is represented by a binary string. The $(\mu+\mu)$ EA for the instance uses the following mutation and selection operators.
\begin{itemize}
\item \emph{Randomised Initialisation:} generate $\mu$ feasible solutions (individuals) at   random.
\item \emph{Bitwise Mutation:} given a string $x$, flip each bit independently with flipping probability $1/n$. If an individual generates an infeasible offspring, the offspring is rejected immediately, while the parent is automatically transferred into the intermediate population of children after mutation.

\item \emph{Elitist Selection:} any elitist selection operator will do.
\end{itemize}

The fitness landscape associated with the $(1+1)$ EA is  non-bridgeable.
According to Proposition~\ref{proNon-bridgeable}, superlinear $\rho$-scalability never happens, meaning that using a population does not increase the average convergence rate.
Similar results have been established before in terms of the expected number of generations it takes to reach the optimum for the One-Max problem. \citet{sudholt2011general}  proved that the (1+1) EA is the best EA to tackle this problem. \citet{jansen2005choice} also analysed the relationship between the runtime and population size but under the big $O$ notation.

\section{Case Study 2: Superlinear $\rho$-Scalability on Bridgeable Fitness Landscapes}
\label{secCaseStudy2}

\subsection{Superlinear   $\rho$-Scalability May Happen on Certain Bridgeable Fitness Landscapes}

\begin{definition}
Given a fitness function $f(x)$, we say that the fitness landscape associated with a $(1+1)$ is \emph{bridgeable} if there exit two non-optimal states $x$ and $y$ where $x$ has a better fitness than  $y$ (that is $f(x) \ge f(y)$) while the probability of entering a higher fitness level $\mathcal{S}^{(1)}_{\high}(x)$ starting from the state $x$ is not less than that starting from $y$, that is
\begin{align*}
 P_M(x, \mathcal{S}^{(1)}_{\high}(x)  )
\le P_M(y, \mathcal{S}^{(1)}_{\high}(x)).
\end{align*}
\end{definition}

Proposition~\ref{proBridgeable} below investigates a particular scenario where ``roads through bridge'' exist on bridgeable fitness landscapes, thereby demonstrating that the use of a population may be helpful when coping with bridgeable fitness landscapes in the sense that superlinear scalability could be achieved under certain conditions.

\begin{proposition}
\label{proBridgeable} Suppose we are given a $(1+1)$ elitist EA and a $(\mu+\mu)$ elitist EAs (where $\mu \ge 2$) that exploit identical mutation operator to maximize the same fitness function. Suppose that the  fitness landscape associated with $(1+1)$ EA is bridgeable. Assume further, that the following conditions hold:
\begin{enumerate}
    \item  \emph{Fitness diversity preservation:} given  $\Phi_t= X$ and $\Phi_{t+1/2}=Y$, if there exists one or more individuals in $X$ or $Y$ whose fitness is less than that of the best individual of $X$, then at least one of these individuals must be selected into the next population with positive probability.

    \item \emph{Existence of bridgeable points:} let $x_{\rho}$ be a state at the 2nd highest fitness level. We require that
    \[
P(x_{\rho},x_{\rho})= \max_{z \in \mathcal{S}^{(1)}_{\non}}  P(z,z)  .
\] All other states  at   lower fitness levels   are bridgeable points of $x_{\rho}$. The probability of going from a bridgeable point $y$ to the optimal set via mutation is  larger than that from $x_{\rho}$ by a factor of $\mu$:
    \begin{align}
 P_M(y, \mathcal{S}^{(1)}_{\opt}    )
  \ge    \mu P_M(x, \mathcal{S}^{(1)}_{\opt}) .\label{equOneMax1}
    \end{align}

    \item \emph{Pass through bridgeable points:} The probability of going from the $x_{\rho}$ above to the set of  bridgeable  points via mutation is  large enough in the following sense:
     \begin{align}
   P_M(x_{\rho}, \mathcal{S}^{(1)}_{\bridge}(x_{\rho}) )
  \ge    \mu P_M(x_{\rho}, \mathcal{S}^{(1)}_{\opt} ) . \label{equOneMax2}
    \end{align}

\end{enumerate}
Then  superlinear $\rho$-scalability happens for such a $\mu$.
\end{proposition}

\begin{proof}
For the $(\mu+\mu)$ EA, let $X$ be any population in the non-optimal set. Now consider the probability of going from   $X$ to the optimal set in two generations. It is convenient to analyze two complementary cases according to the different types of the population $X$.

\paragraph*{Case 1} The population $X = (x_{\rho}, \, x_{\rho}, \ldots, x_{\rho})$ consists of the repeated copies of the fittest individual $x_{\rho}$: .

From Conditions (\ref{equOneMax1}) and (\ref{equOneMax2}), the probability of going from $X$ to  the optimal set in two generations is greater than
\begin{align*}
&P(\Phi_{t+2}  \in \mathcal{S}^{(\mu)}_{\opt} \mid \Phi_{t+1} \in \mathcal{S}^{(\mu)}_{\bridge} (x_{\rho})) \cdot
   (P(\Phi_{t+1} \in \mathcal{S}^{(\mu)}_{\bridge} (x_{\rho}) \mid \Phi_t=X  )   \\
&\ge  \mu^2   (P (x_{\rho}, \mathcal{S}^{(1)}_{\opt} ))^2 \\
&=  \mu^2   (1-P(x_{\rho}, x_{\rho}))^2.
\end{align*}

\paragraph*{Case 2} A population $X =(x_1, \cdots, x_{\mu})$ that contains at least one bridgeable point of $x_{\rho}$ (recall condition 2).

From Condition (\ref{equOneMax2}), the probability of going from $X$ to the optimal set in two generations is greater than
\begin{align*}
&P(\Phi_{t+2}  \in \mathcal{S}^{(\mu)}_{\opt} \mid \Phi_{t+1} \in \mathcal{S}^{(\mu)}_{\opt}) \cdot P(\Phi_{t+1} \in \mathcal{S}^{(\mu)}_{\opt} \mid \Phi_t=X )   \\
&\ge    \mu (P (x_{\rho},   \mathcal{S}^{(1)}_{\opt} ))\\
&= \mu   (1-P(x_{\rho}, x_{\rho}))\\
&\ge   \mu^2   (1-P(x_{\rho}, x_{\rho}))^2.
\end{align*}

Thus, after examining the two mutually exhaustive cases above, we deduce that for all populations $X$ in the non-optimal set,
\begin{align*}
 P(\Phi_{t+2}  \in \mathcal{S}^{(\mu)}_{\opt} \mid \Phi_t =X) &\ge   \mu^2   (1-P(x_{\rho}, x_{\rho}))^2\\
&>   1- \left( 1- \mu (1-P(x_{\rho},x_{\rho}) \right)^2.
\end{align*}

The inequality above shows that the  \emph{road through bridge}  condition (\ref{equOneMax6}) holds implying that the superlinear scalability takes place thanks to theorem~\ref{theSuffNeceElitist}.
\end{proof}

\subsection{An Example of Bridgeable Fitness Landscapes}\label{bridgeableExampleSubsect}
\label{exaBridgeable}
Consider another instance of the average capacity 0-1 knapsack problem: the value of the item $v_1=n$ while the remaining items have values $v_i=1$ for $i=2, \cdots, n$; the weight of the item $w_1=n-1$, while the weights of the remaining items $w_i=1$ for $i=2, \cdots, n$. The capacity $C = n-1$.
The fitness function resembles a fully deceptive function \citep{he2002individual}.
\begin{equation}
\label{equBridgeable}
f_2(x)=\left \{
\begin{array}{llll}
n, & \mbox{if } s_1=1, s_2=\cdots =s_n=0;\\
 \sum^{n}_{i=2}  s_i, &\mbox{if } s_1=0;  \\
\mbox{infeasible}, &\mbox{otherwise}.
\end{array}
\right.
\end{equation}

 An individual is represented by a binary string. The $(\mu+\mu)$ EA for the instance uses the following mutation and selection operators.
\begin{itemize}
\item \emph{Randomised Initialisation:} generate $\mu$ feasible solutions (individuals) at   random.

\item \emph{Bitwise Mutation:} given a binary string $x$, flip each bit independently with probability $1/n$.  If an individual generates an infeasible offspring, the offspring is rejected immediately and the parent is transferred into in the population of children.

\item   \emph{Elitist Proportional Selection: }  the best individual is replaced if the best child individual is fitter, while the non-best individuals are selected from the two populations $X$ and $Y$ (disregarding the best individual) via fitness proportional selection.
\end{itemize}

Notice that the self-transition probability, $P(x,x)$, is maximal when $x=(0,1,\cdots, 1)$. The unique optimal solution having the highest fitness level, $n$, is $(1,0,\cdots, 0)$.  The unique solution having the second highest fitness level, namely $n-1$, is $x=(0,1,\cdots, 1)$. All the other feasible individuals (solutions) are bridgeable points of $x$.

The event of going from   $x=(0,1,\cdots, 1)$  to $(1,0, \cdots, 0)$ via mutation occurs if and only if all of the bits are flipped. The probability of this event happening is
$$
 P_M(x, \mathcal{S}^{(1)}_{\opt} )= \left(\frac{1}{n}\right)^n.
$$

The event of going from  any other feasible state  $y$  (except $(1,0,\cdots, 0)$ and $(0,1,\cdots, 1)$) to $(1,0, \cdots, 0)$ happens if and only if the first bit is flipped,  all the other $1$-valued bits are flipped and all the other $0$-valued bits are unchanged. Let $\mid y\mid$ denote the number of 1-valued bits in $y$.   Since $y$  is a feasible solution but  except $(1,0,\cdots, 0)$ and $(0,1,\cdots, 1)$ so that $\mid y \mid < n-1$. The probability of this event happening  is:
$$   P_M(y, \mathcal{S}^{(1)}_{\opt} ) =  \left( 1-\frac{1}{n} \right)^{n-\mid y \mid-1} \left(\frac{1}{n}\right)^{\mid y \mid+1}.$$

The event of going from the only individual $x=(0,1,\cdots, 1)$ at the second highest fitness level to the set of bridgeable points happens if and only if the first bit is not flipped and at least one of the other bits is flipped. The probability  of this event is then
\begin{align*}
 P_M(x,  \mathcal{S}^{(1)}_{\bridge} (x))=
  \left( 1-\frac{1}{n} \right)   \sum^{n-1}_{k=1} {n-1 \choose k} \left(\frac{1}{n}\right)^k \left( 1-\frac{1}{n} \right)^{n-k-1}.
\end{align*}

Thus, conditions (\ref{equOneMax1}) and (\ref{equOneMax2}) hold for any population size $\mu\le n$ implying that  superlinear $\rho$-scalability happens for  $\mu\le n$.

\section{Conclusions and Discussions}
\label{secConcluion}
\subsection{Conclusions}
A novel approach, based on the fundamental matrix   of absorbing Markov chains, is introduced to study population scalability of EAs.
The spectral radius and matrix norms of the fundamental matrix are used as  a  measure of the performance  of an EA. The reciprocal of the spectral radius  $1/\rho(\mathbf{N})$  is the  average convergence rate interpreted through the notion of the geometric mean. The $\infty$-norm  $\parallel \mathbf{N} \parallel_{\infty}$ is the maximum value of the expected number of generations to encounter an optimum solution for the first time. The $a$-norm $\parallel \mathbf{N} \parallel_a$ is the average value of the expected number of generations to encounter an optimum solution for the first over all transient initial states. Three different notions of population scalability are proposed in the paper: $\rho$-scalability (based on the spectral radius of the fundamental matrix), $\infty$-scalability  (based on the $\infty$-norm  of the fundamental matrix) and $a$-scalability (based on the $a$-norm   of the fundamental matrix).

 The main results of the paper may be summarized in two parts.
\begin{enumerate}

\item Theorem~\ref{theScalability} shows that  $\rho$-scalability always happens  for elitist EAs using global mutation. For a population-based EA  using identical mutation, the average convergence rate of a $(\mu+\mu)$ EA (where $\mu\ge2$) is always larger than that of the $(1+1)$ EA. Nonetheless, $a$-scalability and $\infty$-scalability may not take place. Using a larger population size sometimes will increase, rather than reduce, the expected number of generations to encounter an optimal solution for the first time (when measured either in terms of the maximum value or the average value). This fact is counterintuitive to the commonly accepted ``rule of thumb" in evolutionary computation.

\item Theorems~\ref{theSuffNece}, \ref{theSuffNeceElitist} and~\ref{theBridgeNecessary} provide  sufficient and/or necessary conditions for the superlinear $\rho$-scalability to take place.
The conditions indicate that for the elitist EAs optimizing the same fitness function and using identical mutation operators, the average convergence rate of a $(\mu+\mu)$ EA (where $\mu\ge2$) is more than $\mu$ times that of the corresponding $(1+1)$ EA if and only if the probability of passing through the ``roads through bridge'' is sufficiently large for the $(\mu+\mu)$ EA.
\end{enumerate}

In order to illustrate the theoretical findings above, two cases studies are provided in the paper. The first case study shows that the average convergence rate of a $(\mu+\mu)$ EA (where $\mu\ge2$) is never larger than $\mu$ times that of the $(1+1)$ EA on any non-bridgeable fitness landscapes. The second one illustrates that the average convergence rate of a $(\mu+\mu)$ EA (where $\mu\ge 2$) might be larger than $\mu$ times that of the $(1+1)$ EA on certain bridgeable fitness landscapes.

The notion of population scalability is not intended to compare the performance of the corresponding $(1+1)$ and $(\mu + \mu)$ EAs on all instances of a given combinatorial optimization problem such as, for instance, the 0-1 knapsak problem, at once. Indeed, as we have seen in Subsections~\ref{nonBridgeableSubsectExample} and \ref{bridgeableExampleSubsect}, for the same pair of the corresponding $(1+1)$ and $(\mu + \mu)$ EAs, population scalability may take place on one instance, and, simultaneously, not on another instance of the 0-1 knapsak problem, so that it is meaningless to consider population scalability on all instances at once.

While the approach based on the notion of the fundamental matrix is rather virtuous for analysing and understanding the notion of population scalability of EAs, it is unlikely to be practical when it comes to calculating $\rho$-population scalability for a specified pair of EAs optimizing a given fitness function. This is due to the fact that the fundamental matrix is usually difficult to compute. Likewise the calculation of $\infty$-scalability and $a$-scalability is not an easy job.

There are still  many open questions some of which are listed below. Can we pin down any insightful sufficient and/or necessary conditions that the expected number of generations that a $(\mu+\mu)$ EA (where $\mu \ge 2$) takes to encounter an optimal solution for the first time is greater than that the corresponding $(1+1)$ EA (measure by either the average value or maximum value) does? Can we address the same question in terms of the superlinear scalability? How to determine the threshold of the population size when an EA loses its superlinear scalability? Is there any feasible approach to calculating or, at least, estimating the population scalability?

 \subsection{Discussions: Other Types of EAs}

The condition  that the EAs are convergent is not necessary. If an EA is not convergent, then $\rho(\mathbf{Q}^{(\mu)})=1$. In this case, the corresponding notion of the population scalability is revised as follows:
\begin{equation}
\mbox{$\rho$-scalability($\mu$)} =
\left\{ \begin{array}{lll}
\frac{\rho(\mathbf{N}^{(1)})}{\rho(\mathbf{N}^{(\mu)})}, &\mbox{if } \rho(\mathbf{Q}^{(1)})<1,\\
+\infty, &\mbox{if } \rho(\mathbf{Q}^{(1)})=1, \rho(\mathbf{Q}^{(\mu)})<1,\\
\mbox{indefined}, &\mbox{if } \rho(\mathbf{Q}^{(1)})=1, \rho(\mathbf{Q}^{(\mu)})=1.
\end{array}
\right.
\end{equation}

 If a mutation operator is not global, we can easily revise it exploiting mixed strategy~\citep{he2012pure}. We apply this mutation operator with the probability $1-\epsilon$ and apply a global mutation with the probability $\epsilon$ for a  small $\epsilon>0$. Evidently, the mixed strategy mutation operator obtained in this manner is global.

Crossover is widely used in EAs. Since a $(1+1)$ EA  doesn't include any crossover operator, it is not an appropriate candidate as the benchmark EA. Instead, a $(2+2)$ EA with  crossover would play the role of the benchmark EA. Thus, the notion of scalability would then be revised accordingly:
\begin{equation}
\mbox{$\rho$-scalability($\mu$)}=\frac{\rho (\mathbf{N}^{(2)})}{\rho(\mathbf{N}^{(\mu)})}.
\end{equation}

The EAs above can still be modelled by absorbing Markov chains and Theorem~\ref{theSuffNece}   is applicable.

Nonetheless, it seems rather difficult to apply the fundamental matrix approach to EAs that exploit time-dependent genetic operators.
\section*{Acknowledgements}
We would especially like to thank Professor G\"unter Rudolph who has been involved in initializing the research issue of population scalability.   This work is supported by  the EPSRC under Grant EP/I009809/1 and National Natural Science Foundation of China under Grant 61170081.

\end{document}